\documentclass[letterpaper]{article} 
\usepackage{aaai24}  
\usepackage{times}  
\usepackage{helvet}  
\usepackage{courier}  
\usepackage[hyphens]{url}  
\usepackage{graphicx} 
\usepackage{natbib}  
\usepackage{caption} 
\usepackage{algorithm}
\usepackage{algorithmic}

\usepackage{microtype}
\usepackage{graphicx}
\usepackage{subfigure}
\usepackage{booktabs} 

\usepackage{mathtools}
\usepackage{enumitem}
\usepackage{multirow}
\usepackage{rotating}
\usepackage{color}
\usepackage[capitalize,noabbrev]{cleveref}
\usepackage{multicol}
\usepackage{blindtext}
\usepackage{svg}
\usepackage[textsize=tiny]{todonotes}
\usepackage{caption}

\usepackage[utf8]{inputenc} 
\usepackage[T1]{fontenc}    
\usepackage{url}            
\usepackage{booktabs}       
\usepackage{amsfonts}       
\usepackage{nicefrac}       
\usepackage{microtype}      
\usepackage{xcolor}         
\usepackage{amsmath}
\usepackage{amssymb}
\usepackage{amsthm}

\usepackage{newfloat}
\usepackage{listings}
\DeclareCaptionStyle{ruled}{labelfont=normalfont,labelsep=colon,strut=off} 
\floatstyle{ruled}
\newfloat{listing}{tb}{lst}{}
\floatname{listing}{Listing}
\usepackage{amsfonts,amssymb,mathrsfs,float}
\usepackage{graphicx}
\usepackage{algorithm,algorithmic}
\usepackage{color}
\usepackage{multicol}
\usepackage{multirow}
\usepackage{appendix}

\theoremstyle{plain}
\newtheorem{theorem}{Theorem}[section]

\newtheorem{lemma}[theorem]{Lemma}

\theoremstyle{definition}
\newtheorem{definition}[theorem]{Definition}

\theoremstyle{remark}
\newtheorem{remark}[theorem]{Remark}

\def\I{{\mathbb{I}}}

\def\R{{\mathbb{R}}}

\def\A{{\mathcal{A}}}
\def\B{{\mathcal{B}}}
\def\F{{\mathcal{F}}}

\def\D{{\mathcal{D}}}

\def\S{{\mathcal{S}}}

\def\sign{{\hbox{\rm{sign}}}}

\def\Prob{{\hbox{\rm{Prob}}}}
\def\Acc{{\hbox{\rm{Acc}}}}

\def\CE{{{\rm{CE}}}}

\def\tr{{{\rm{tr}}}}
\def\ts{{{\rm{te}}}}
\def\poi{{{\rm{po}}}}
\def\val{{{\rm{va}}}}
\def\rb{{{\rm{rb}}}}

\def\sep{{\,\|\,}}

\def\text#1{\rm{#1}}

\title{Detection and Defense of Unlearnable Examples}
\author{%
  Yifan Zhu$^{1, 3}$, Lijia Yu$^{2, 3}$, Xiao-Shan Gao$^{1, 3}$
    \thanks{Corresponding author.}}

\affiliations{
\textsuperscript{\rm 1}Academy of Mathematics and Systems Science, Chinese Academy of Sciences, Beijing 100190, China\\
\textsuperscript{\rm 2}Institute of Software, Chinese Academy of Sciences, Beijing 100190,  China \\
\textsuperscript{\rm 3}University of Chinese Academy of Sciences, Beijing 101408, China\\
\footnotesize{{zhuyifan@amss.ac.cn}}, 
\footnotesize{{yulijia@ios.ac.cn}}, 
\footnotesize{{xgao@mmrc.iss.ac.cn}}
}

\begin{document}
\maketitle
\begin{abstract}
Privacy preserving has become increasingly critical with the emergence of social media. Unlearnable examples have been proposed to avoid leaking personal information on the Internet by degrading generalization abilities of deep learning models. However, our study reveals that unlearnable examples are easily detectable. We provide theoretical results on linear separability of certain unlearnable poisoned dataset and simple network based detection methods that can identify all existing unlearnable examples, as demonstrated by extensive experiments. 
Detectability of unlearnable examples with simple networks motivates us to design a novel defense method. We propose using stronger data augmentations coupled with adversarial noises generated by simple networks, to degrade the detectability and thus provide effective defense against unlearnable examples with a lower cost.
Adversarial training with large budgets is a widely-used defense method on unlearnable examples. We establish quantitative criteria between the poison and adversarial budgets which determine the existence of robust unlearnable examples or the failure of the adversarial defense. 
\end{abstract}

\section{Introduction}
Deep neural networks (DNNs)  have become the most powerful machine learning method, driving significant advances across numerous fields.
However, the security of deep learning remains a major concern.
One of the most serious security threat is {\em data poisoning}, where an attacker can manipulate the training data in a way that intentionally causes deep models to malfunction.
For example, by injecting triggers during the training phase, {\em backdoor attacks} \cite{chen2017targeted,gu2019badnets} will cause malfunctions or enable attackers to achieve specific objectives when triggers are activated.
Another type of data poisoning attack is {\em availability attacks} \cite{biggio2012poisoning,koh2017understanding,lu2022indiscriminate}, aiming to reduce the generalization capability of deep learning models by modifying features and labels of the training data.
Recently, {\em unlearnable data poisoning attacks} \cite{huang2020unlearnable,wang2021fooling} were proposed, which can modify features of all training data with a small poison budget to generate {\em unlearnable examples}.
Privacy preserving has become increasingly eye-catching in recent years, especially on face recognition \cite{shan2020fawkes, cherepanova2020lowkey, hu2022protecting}.
Unlearnable examples were designed initially for the ``let poisons be kind'' \cite{huang2020unlearnable,wang2021fooling} approach to privacy preservation,
allowing individuals to slightly modify their personal data to make them unlearnable without losing the semantics.

While unlearnable examples have been shown to be highly effective at deceiving victims into thinking that their deep learning models have been trained successfully by achieving high validation accuracy, this paper demonstrates that unlearnable examples can be detected with relative ease.
Specifically, we propose two effective methods to detect whether a given dataset has been poisoned by unlearnable data poisoning attacks.
It was experimental observed in \cite{yu2021indiscriminate} that unlearnable poisons  are linearly separable. 
In this paper, we further prove that for certain random or region class-wise poisons, the poisoned datasets rather than poisons  are linearly separable if the dimension of data is sufficiently large.
Based on these theoretical findings, we propose the {\em Simple Networks Detection} algorithm, which leverages linear models or simple two-layer networks to detect poisoned data. 
Additionally, We have experimentally observed that poisons are immune to large bias shifts. 
Based on this observation, we propose another detection algorithm called {\em Bias-shifting Noise Test}, which introduces a large bias to each training data, to destroy their original features while retaining the injected poison features. 
The resulting difference can be used to detect existences of poisons.
Our experiments on CIFAR-10, CIFAR-100, and TinyImageNet  demonstrate that all major unlearnable examples can be  effectively detected by both of the algorithms.

Furthermore, the detectability of unlearnable examples has inspired us to develop a defense strategy that focuses on degrading them. 
We achieve this through the use of data augmentations and adversarial noises that make these examples undetectable by simple networks.
By using stronger data augmentations and a two-layer neural network (NN) to generate stronger noises, we demonstrate that it becomes much more difficult to generate unlearnable examples.
Experiments have shown that our defense method is highly effective against unlearnable examples, which can even outperforms adversarial training regimes and state-of-the-art defense methods.

Adversarial training with  large defense budgets is a well-known defense method against unlearnable examples \cite{tao2021better}.
%
We establish theoretical criteria for the relationship between the poison budget and the adversarial defense budget.
Specifically, we prove that if the poison budget exceeds four times the size of the adversarial defense budget, or if the gap between the two budgets exceeds a certain constant, then robust unlearnable examples can be created to make adversarial training set linearly separable.
These theoretical results guarantee the existence of robust unlearnable examples, while also providing a lower bound on the adversarial defense budget required for adversarial training to be effective.
We also prove that to make dataset linearly separable, poison budget must be larger than a certain constant, under some mild assumptions.
Our experimental results confirm the validity of the theoretical results, showing that as the adversarial budget increases, unlearnable examples become learnable again, because the linear features injected by attackers are destroyed through adversarial training.

\vskip 5pt

We summarize our main contributions as follows:
\begin{itemize}
\item We propose two effective methods to detect whether a dataset has been poisoned by unlearnable data poisoning attacks, based on  theoretical results and experimental observations.

\item
We demonstrate that stronger data augmentations with adversarial noises generated by a simple network can destroy the detectability, as well as achieve good defense performance.

\item
We establish certified upper bounds of the poison budget relative to the adversarial defense budget required for generating robust unlearnable examples. 
\end{itemize}

\section{Related work}

\paragraph{Data Poisoning.} Data poisoning is a type of attack that can cause deep learning models to malfunction by modifying the training dataset.
Targeted data poisoning attacks 
\cite{shafahi2018poison,zhu2019transferable,huang2020metapoison,schwarzschild2021just} aimed to misclassify specific targets.
Previous availability attacks \cite{biggio2012poisoning,munoz2017towards} attempted to reduce test performance of models by poisoning a small portion of the training data. Unlearnable attacks \cite{huang2020unlearnable} were a special type of availability attack where the attacker poisons all of the training data using a small poison budget, resulting in significantly drop in test accuracy to almost random guessing \cite{feng2019learning, huang2020unlearnable,fowl2021adversarial,sandoval2022autoregressive}.
Robust unlearnable attacks \cite{fu2021robust} attempted to maintain the poisoning effects under adversarial training regimes. 
Another type of data poisoning was backdoor attacks \cite{chen2017targeted,gu2019badnets,barni2019new,turner2019label}, which induced triggers into trained models to cause malfunctions.

\paragraph{Attack Detection.} 
In recent years, several methods have been proposed to detect safety attacks of DNNs. Detection of adversarial attacks has been explored in \cite{metzen2016detecting,grosse2017statistical,abusnaina2021adversarial} for victims to determine whether a given data is an adversarial example or not.
Detection algorithms were provided to identify triggers and recover them under backdoor attacks \cite{chen2018detecting, chen2019deepinspect, dong2021black}.
\citet{guo2021aeva} detected backdoor attacked model using adversarial extreme value analysis.
Our detection methods are the first one to focus on unlearnable examples to the best of our knowledge.

\paragraph{Poison Defense.} Several defense methods for data poisoning have  emerged in recent years. In \citet{ma2019data}, a defense method using differential privacy was proposed.
In \citet{chen2021pois}, generative adversarial networks were used to detect and discriminate poisoned data. 
\citet{liu2022friendly} proposed using friendly noise to improve defense against data poisoning. 
\citet{wang2021robust} analyzed the robustness of stochastic gradient descent on data poisoning attacks.
 People also used adversarial training to defend unlearnable data poisoning attacks \cite{tao2021better}.
Recently, defense methods on unlearnable examples has been provided \cite{qin2023learning, liu2023image}.

\section{Notations and definitions of unlearnable examples}
\label{poison}

\paragraph{Notations.}
Denote the training dataset as $D_{\tr}=\{(\boldsymbol{x}_i,y_i)\}_{i=1}^{N}\subset\I^d\times[C]$,
where $d,N,C$ are positive integers, and $\I=\left[0,1\right]$ is the value range of data, $[C]=\{1,\ldots,C\}$ is the set of labels.
Let $D_{\tr}^{\poi}=\{(\boldsymbol{x}_i+\boldsymbol{\epsilon}(\boldsymbol{x}_i),y_i)\}_{i=1}^{N}$ be the {\em poisoned training dataset} of $D_{\tr}$,
where $\boldsymbol{\epsilon}(\boldsymbol{x}_i)$ is the poison elaborately generated by the poison attacker. In this paper, we assume that $\boldsymbol{\epsilon}(\boldsymbol{x}_i)$ is a small perturbation that satisfies $||\boldsymbol{\epsilon}(\boldsymbol{x}_i)||_{\infty}\leq\eta$,
where $\eta\in\R_{>0}$ is called the {\em poison budget}. 
%
A well-learned network $\F$ has good generalization performance on the test set $D_{\ts}$, i.e., have high test accuracy denoted as $\textup{Acc}(\F,D_{\ts})$.

\paragraph{Unlearnable examples.}
A poisoned dataset $D_{\tr}^{\poi}$ is called {\em unlearnable}~\cite{huang2020unlearnable},
if training a network $\F$ on $D_{\tr}^{\poi}$ results in
very low test accuracy $\textup{Acc}(\F,D_{\ts})$,
but achieves sufficiently high (poisoned) validation accuracy $\textup{Acc}(\F,D_{\textup{val}}^{\poi})$ on poisoned validation dataset. 
%
%
If the victim receives an unlearnable poisoned dataset, they may split it into a training set and a validation set. 
 Since the (poisoned) validation accuracy is good enough, the victim may assume that their model is performing well.
However, because the victim has no access to the (clean) test set $D_{\ts}$, their model will actually perform poorly on $D_{\ts}$. As a result, the victim will be deceived by the poisoned generator.

There are two basic types of unlearnable poisoning attacks:
{\em sample-wise} and {\em class-wise},
where sample-wise means that each sample is independently poisoned with a specific perturbation and class-wise means that samples with the same label are poisoned with the same perturbation.
 Some examples of poisoned data and their corresponding perturbations can be found in Appendix \ref{poison-img}.
 
The main unlearnable attack methods include:
Random(C), Region-$n$~\cite{sandoval2022poisons}, Err-min~\cite{huang2020unlearnable}, Err-max~\cite{fu2021robust}, NTGA~\cite{yuan2021neural}, AR~\cite{sandoval2022autoregressive}, RobustEM~\cite{fu2021robust}, CP~\cite{he2022indiscriminate}, TUE~\cite{ren2022transferable}, etc.

\section{Detection of unlearnable examples}

\label{section-detection}
Although unlearnable examples can achieve their goal of deceiving victims by having good validation accuracy and poor test accuracy in normal training processes \cite{huang2020unlearnable},
 we will demonstrate in this section that unlearnable examples can actually be easily detected. We will also provide   effective detection algorithms for detecting unlearnable examples using simple networks.

\subsection{Theoretical analyses of unlearnable examples}
\label{section-detection-theo}

Firstly, we present theoretical results that certain unlearnable examples are linearly separable, which can be used to identify the presence of poisons.
In \cite{yu2021indiscriminate}, they empirically discover the linear separability of unlearnable poison noises,
tt is worth noting that our approach differs from previous work \cite{yu2021indiscriminate}, in that we not only empirically find but also prove that unlearnable poisoned dataset, rather than noises, are linearly separable.

\begin{theorem}
\label{th-ex1}
Let $D=\left\{(\boldsymbol{x}_i,y_i)\right\}_{i=1}^N\subset\I^d\times[C]$. For the class-wise poison $\left\{\boldsymbol{v}_i\right\}_{i=1}^C\subset \R^d$ satisfying that $\forall i\in\left[C\right]$ and $j\in\left[d\right]$, 
$(\boldsymbol{v}_i)_j$ is i.i.d. and obeys distribution $\Delta(\epsilon)$, where $\Delta(\epsilon)=2\epsilon\cdot \textup{Bernoulli}\left(\frac{1}{2}\right)-\epsilon$, 
that is, $(\boldsymbol{v}_i)_j$ equals $\pm\epsilon$  with $\frac{1}{2}$ probability respectively.
Then with probability at least
$1-NC\left(2e^{-\frac{d\epsilon^2}{18}}+e^{-\frac{d}{32}}\right)$,
the class-wise poisoned dataset $D^{\poi}=\left\{\boldsymbol{x}_i+\boldsymbol{v}_{y_i},y_i\right\}_{i=1}^N$
is linearly separable.
\end{theorem}

\begin{theorem}
\label{th-ex1-bc}
For the Region-$k$ poison $\left\{\boldsymbol{v}_i\right\}_{i=1}^C\subset \R^d$ satisfies that $\forall i\in\left[C\right]$, $(\boldsymbol{v}_i)_j$ is equal whenever $j$ is in the same region;
$(\boldsymbol{v}_i)_j$ is i.i.d. and obeys distribution $\Delta(\epsilon)$, whenever $j$ in different regions.
Then with probability at least
$1-NC\left(2e^{-\frac{k\epsilon^2}{18}}+e^{-\frac{k}{32}}\right)$, the Region-$k$ poisoned dataset $D^{\poi}=\left\{\boldsymbol{x}_i+\boldsymbol{v}_{y_i},y_i\right\}_{i=1}^N$ is linearly separable.
\end{theorem}

The proofs of Theorems \ref{th-ex1} and \ref{th-ex1-bc} are deferred to Appendices \ref{proofs-1} and \ref{proofs-2} respectively.

\begin{remark}
By Theorem \ref{th-ex1}, when $d$ or $\epsilon$ are sufficiently large, certain Random(C) poisoned dataset is linearly separable.
Similarly, by Theorem \ref{th-ex1-bc},when $k$ is sufficiently large, certain Region-$k$ poisoned dataset is linearly separable. 
\end{remark}

\paragraph{Magnitude of poison budget to achieving linear separability.}
\label{re-epsilon-sep}
We add the Random(C) poison $\boldsymbol{v}_i\in\{+\epsilon,-\epsilon\}^d$ to CIFAR-10 with $\epsilon=8/255$, and use the linear network $\F(\boldsymbol{x})=[\boldsymbol{v}_1,\cdots,\boldsymbol{v}_C]^T\boldsymbol{x}$ for classification. Experimental results show that $49,963$ of $50,000$ poisoned training samples can be correctly classified by $\F$.  This finding indicates that $\epsilon=8/255$ is large enough to achieve linear separability for CIFAR-10 dataset.
Table \ref{tab-diff-poison} also shows that Region-16 and Err-min(S) poisons with budget $\epsilon=8/255$ is enough to make poisoned dataset linearly separable.

\paragraph{Sample-wise poisons.}
Theorems \ref{th-ex1} and \ref{th-ex1-bc} describe properties of class-wise poisons,but sample-wise poisons have similar properties as well.
In Appendix \ref{sec-7.3}, we provide experiments to demonstrate  the similarity between sample-wise and class-wise poisons by measuring the cosine similarity and commutative KL divergence.
Additionally, from Appendix \ref{sec-obs}, we can observe that both sample-wise and class-wise error-minimum poisoned dataset have similar training curves.

\subsection{Detection of unlearnable examples by simple networks}
\label{section-detection-simp}
Theorems \ref{th-ex1} and \ref{th-ex1-bc} imply that a poisoned dataset can be learned by a linear network. 
However, the clean dataset such as CIFAR-10, CIFAR-100, and TinyImageNet cannot be easily fitted by a linear model. 
We evaluate the linear separability rate for these datasets as shown below.

\begin{definition}[Linear Separability Rate]
Let dataset $S\subset\I^d\times[C]$ and $\F_{\textup{linear}}$ denotes the set of all linear models $f:\mathbb{R}^d\to \mathbb{R}^C$.
The {\em linear separability rate} of  $S$ is defined as
$\beta_S=\sup\limits_{f\in\F_{\textup{linear}}}\textup{Acc}(f,S).$
\end{definition}

\begin{remark}
\label{rem-58}

The linear separability rates of CIFAR-10, CIFAR-100, and TinyImageNet are 
$46.53\%$, $31.71\%$, $49.38\%$, respectively, 
by training them with a linear network.
\end{remark}

This difference between a clean dataset and a poisoned one can be  exploited to detect the presence of unlearnable examples, which motivates Algorithm \ref{alg-weak}, named  {\em Simple Networks Detection}.

\begin{algorithm}[!htt]
\caption{Simple Networks Detection}
\label{alg-weak}
\begin{algorithmic}
\renewcommand{\algorithmicrequire}{ \textbf{Input:}}
\REQUIRE
A dataset $D$ might be poisoned. 
A linear network or a two-layer NN (hidden width equals the data dimension) $\F$. 
A detection bound $B$ (say 0.7).
\renewcommand{\algorithmicrequire}{ \textbf{Output:}}
\REQUIRE Poison function $I(D)$.\\
$I(D)=1$ if $D$ is recognized as the poisoned dataset;\\
$I(D)=0$ if $D$ is recognized as the clean dataset.\\
%
\renewcommand{\algorithmicrequire}{ \textbf{Do:}}
\REQUIRE
~\\
Initialize parameters of the network $\F$. \\
Train the network $\F$ on dataset $D$ with loss function $L_{\CE}(\F(\boldsymbol{x}),y).$\\
%
\textbf{if} $\Acc(\F,D)\leq B$: $I(D)=0$; \textbf{else} $I(D)=1$.
\end{algorithmic}
\end{algorithm}

\paragraph{Detection under data augmentations.}
It is worth noting that, in practice, people often train the networks with some data augmentation methods. 
For instance, for CIFAR-10, random crop and random horizontal flip are commonly used data augmentation methods.
However, experimental observations have shown that the linear separability of poisoned datasets may easily be broken by data augmentations. 
Therefore, in cases where linear separability does not hold, utilizing a two-layer network emerges as the next most suitable criterion.
Furthermore, we can not further relax the simple networks to three-layer NN. This is because the training accuracy on clean data becomes excessively high, making it challenging to distinguish clean datasets from poisoned ones.

\subsection{Detection of unlearnable examples by bias-shifting noise}
\label{section-detection-bs}
In this section, we find out that unlearnable examples are almost immune to large bias-shifting noises. 
Inspired by this unusual behavior of unlearnable examples, we can detect poisons by training it with bias-shifting noises.

\paragraph{Resistance of poisoned dataset to bias-shifting noise.}

Table \ref{tab-noise-learn} in Appendix \ref{sec-7.2} shows that injected unlearnable poisons are strong features dominating the poisoned data, and the poisoned dataset are highly robust to bias-shifting noise.
In unlearnable examples setting, the injected poisons are very small, restricting to less than $8/255$ under the $l_{\infty}$ norm.
However, even when subjected to bias-shifting noise of dozens of times larger   like $\pm 0.5$, the victim model can still achieve $100\%$ (validation) accuracy. 

\begin{table}[!htbp]
\caption{Validation accuracy (\%) on different dataset:
$D_{\textup{clean}}=\{(\boldsymbol{x}_i,y_i)\}$ is the clean validation set,
$D_{\textup{poi}}=\{(\boldsymbol{x}_i+\boldsymbol{\epsilon}(\boldsymbol{x}_i),y_i)\}$,
$D_{\textup{poi-shift}}^b=\{(\boldsymbol{x}_i+\boldsymbol{\epsilon}(\boldsymbol{x}_i)+b\boldsymbol{e},y_i)\}$, where $b$ is the bias-shifting noise level and $\boldsymbol{e}$ is the all-ones vector.
}

\label{tab-learning-noise11}
\centering
\setlength{\tabcolsep}{2.3pt}
\begin{tabular}{lccccccc}
\toprule
Poison & Random(C) & Region-16  & Err-min(S) & NTGA & AR\\
\midrule
$D_{\textup{clean}}$ & 10.46 & 15.27  &10.02  &10.10 &13.63 \\
$D_{\textup{poi}}$ & 100.0 & 100.0 & 99.99   & 99.98 & 99.98\\
$D_{\textup{poi-shift}}^{0.5}$ & 99.84 & 99.64  & 100.0  & 95.26 & 99.82\\
$D_{\textup{poi-shift}}^{-0.5}$ & 100.0 & 99.98 & 97.85   & 90.82 & 90.96\\
\bottomrule
\end{tabular}
\end{table}

Table \ref{tab-learning-noise11} shows the accuracy on poisoned (validation) set and dataset with large bias-shifting noise.
When adding large bias-shifting noises to the poisoned dataset, the validation accuracy of it does not degrade significantly.
But when training on the clean dataset, it will drop significantly because the original features are destroyed by large bias-shifting noise.
This observation motivates us to introduce Algorithm \ref{alg-rule}, which is called {\em Bias-shifting Noise Test}.

\begin{algorithm}[!htbp]
\caption{Bias-shifting Noise Test}
\label{alg-rule}
\begin{algorithmic}
\renewcommand{\algorithmicrequire}{ \textbf{Input:}}
\REQUIRE 
A dataset $D$ might be poisoned. A DNN $\F$ (say  ResNet18). A detection bound $B$ (say 0.7). A bias-shifting noise $\boldsymbol{\epsilon}_{b}$.
\renewcommand{\algorithmicrequire}{ \textbf{Require:}}
\renewcommand{\algorithmicrequire}{ \textbf{Output:}}
\REQUIRE poison function $I(D)$.\\
\renewcommand{\algorithmicrequire}{ \textbf{Do:}}
\REQUIRE
~\\
Randomly split $D$ into training and validation set $D_{\tr}$ and $D_{\val}$.\\
Let the bias-shifting training set be $D_{\tr}^{\rb}=\{(\boldsymbol{x}+\boldsymbol{\epsilon}_{b},y)\sep (\boldsymbol{x},y)\in D_{\tr}\}$.\\
Initialize parameters of the network $\F$. \\
Train the network $\F$ on the bias-shifting training set $D_{\tr}^{\rb}$ with loss function $L_{\CE}(\F(\boldsymbol{x}),y).$\\
\textbf{if} $\Acc(\F,D_{\val})\leq B$: $I(D)=0$; \textbf{else} $I(D)=1$.
\end{algorithmic}
\end{algorithm}

\paragraph{Choice of bias-shifting noise.}
If data are lying in the range $[a,b]$, it is recommended to choose $\frac{b-a}{2}\boldsymbol{e}$ or $-\frac{b-a}{2}\boldsymbol{e}$ as the bias-shifting noise $\boldsymbol{\epsilon}_b$.
As images always lie in $[0,1]$, for simplicity we choose $\boldsymbol{\epsilon}_b = \pm 0.5\boldsymbol{e}$ in this paper, more results on different choices of $\boldsymbol{\epsilon}_b$ are provided in Appendix \ref{app-bias-shift}.
Such noise can effectively destroy the original features, without affecting the injected noise features, as shown in
Table \ref{tab-learning-noise11} for dataset  $D_{\text{poi-shift}}^{\pm 0.5}$.

\section{Defense of unlearnable examples by breaking detectability}
\label{strong-aug}
In Section \ref{section-detection}, we have proven that certain unlearnable examples can be separated by linear networks, and shown that all of the existing unlearnable examples can be easily fitted by simple networks like two-layer NNs even under the usual data augmentations regime.
Therefore, we believe properties that can be fitted by simple networks are the principle of why unlearnable examples work, which can also be used for detection.

 Based on this, once the dataset is detected to be unlearnable, one potential solution is to defend it by destroying its detectability. 
 Adversarial training is a well-known approach for defending against unlearnable examples \cite{tao2021better}, but it is expensive. We can achieve similar goals at a lower cost by breaking detectability of unlearnable examples.

\paragraph{Adding hard-to-learn adversarial noises of simple networks degrades detectability.}
To destroy the detectability of simple networks, we may add adversarial noise $\boldsymbol{\epsilon}$ to each $\boldsymbol{x}_i$, which is hard-to-learn for a simple two-layer NN $\F_{\textup{simple}}$.
Adversarial noises $\boldsymbol{\epsilon}(\boldsymbol{x}_i)$ are generated by PGD attack \cite{madry2018towards} on the robustly learned network $\F_{\textup{simple}}^{\textup{robust}}$, where  $\F_{\textup{simple}}^{\textup{robust}}$ is obtained by adversarial training: 
$$\arg\min\limits_{\F_{\textup{simple}}^{\textup{robust}}}\sum\limits_{(\boldsymbol{x}_i,y_i)\in D} \max\limits_{||\boldsymbol{\delta}||\le\eta} Loss(\F_{\textup{simple}}^{\textup{robust}}(\boldsymbol{x}_i+\boldsymbol{\delta}(\boldsymbol{x}_i)),y_i).$$
%
The generateed adversarial noise will make it difficult for poisoned dataset to be fitted by simple networks, which can destroy the detectability of unlearnable examples, while the small budget $\eta$ will not affect the original features.

\paragraph{Stronger data augmentations destroy detectability.}

\begin{table}[!htbp]
\caption{The detection (training) accuracy (\%) of Algorithm \ref{alg-weak} under two-layer NN for poisoned CIFAR-10 and CIFAR-100 with stronger data augmentation used in contrastive learning.}
\label{strong-aug-detect}
\centering
\begin{tabular}{lcccc}
  \toprule
  Poisons & CIFAR-10 & CIFAR-100\\

\midrule
  Clean data & \textit{27.41}  & \textit{10.67} \\
  Region-16 & \textit{35.20} & \textit{16.57}  \\
  Err-min(S) & \textit{27.28}  & \textit{11.25}  \\
  RobustEM & \textit{27.27}  & \textit{11.51} \\
\bottomrule
\end{tabular}
\end{table}

As discussed in Section \ref{section-detection}, we have proven that certain class-wise unlearnable examples are linearly separable, but linear model detection under standard data augmentations will   fail. This inspires us to use a relaxed version of the model, such as a two-layer NN, for effective detection. 
Therefore, when detecting unlearnable examples by simple networks, data augmentation methods may degrade the detectability.

Inspired by the role of data augmentations in detection, we  introduce stronger data augmentation used in contrastive learning \cite{he2020momentum, chen2020simple} which contains random resized crop, random horizontal flip, color jitter and random grayscale.
We conduct experiments on Table \ref{strong-aug-detect} to demonstrate the detection performance under stronger data augmentations. Results show that two-layer NN is hardly to detect whether a dataset is poisoned under stronger data augmentations.
Therefore, we conclude that stronger data augmentations can make it easier to break detectability of unlearnable examples.

Based on the above discussions, we   propose a fast algorithm \ref{alg-strong-aug} which can break detectability of poisoned dataset to defend against unlearnable examples.
Experimental results on Algorithm \ref{alg-strong-aug} will be provided in Section \ref{exp-defense}.

\begin{algorithm}[!htt]
\caption{Stronger Data Augmentations with Adversarial Noises (SDA+AN)}
\label{alg-strong-aug}
\begin{algorithmic}
\renewcommand{\algorithmicrequire}{ \textbf{Input:}}
\REQUIRE
Unlearnably poisoned dataset $D=\left\{(\boldsymbol{x}_i, y_i)\right\}_{i=1}^N$. Two-layer NN $\F_{\textup{simple}}$. Data augmentation method $A$. A DNN $\F$ (Say ResNet18).
\renewcommand{\algorithmicrequire}{ \textbf{Output:}}
\REQUIRE 
Trained network $\F$.

\renewcommand{\algorithmicrequire}{ \textbf{Do:}}
\REQUIRE 
~\\
\quad Initialize parameters of the networks $\F_{\textup{simple}}$ and $\F$. \\
\quad Adversarially train the network $\F_{\textup{simple}}$ on dataset $D$.\\
\quad Generate adversarial noise on adversarially-trained network $\F_{\textup{simple}}^*$ with $\boldsymbol{\epsilon}(\boldsymbol{x}_i)=\arg\max\limits_{\boldsymbol{||\epsilon||}_p\le\eta} L_{\CE}(\F_{\textup{simple}}^*(\boldsymbol{x}_i+\boldsymbol{\epsilon}),y_i).$\\
\quad Train classification network $\F$ with data augmentation $A$: 
$\min\limits_{\F}\sum\limits_{i=1}^N L_{\CE}(A((\boldsymbol{x}_i+\boldsymbol{\epsilon}(\boldsymbol{x}_i)),y_i))$.

\end{algorithmic}
\end{algorithm}

\section{Criteria of the poison and defense budgets under adversarial training}
\label{def-ats}
Adversarial training is a widely-used defense method against unlearnable examples \cite{tao2021better}. 
However, there is a trade-off between accuracy and robustness in adversarial training, choosing a huge budget to resist unlearnable examples will affect accuracy.
Nevertheless, \cite{wang2021fooling,fu2021robust} show experimentally that adversarial training  with small budget may fail to defend some unlearnable attacks, called {\em robust unlearnable examples}.

On the theoretical aspect, \cite{tao2021better}  proved that unlearnable attacks will fail when the adversarial budget is greater than or equal to the poison budget.
In this section, we will prove three theoretical results on criteria between the poison budget and the adversarial defense budget, and in particular give a certified upper bound on the poison budget for the existence of robust unlearnable examples.
First, we give a definition of the adversarial training set, which is the largest set that adversarial training can use for training.

\begin{definition}[Adversarial Training Set]
\label{ats}
Let $S$ be a (clean) training set. 
$\B_p(S,\epsilon)$ is called the {\em adversarial training set} of $S$  with a small budget $\epsilon$, which is defined as follow:
$$\B_p(S,\epsilon) =\left\{(\boldsymbol{x}+\boldsymbol{\delta},y)\sep (\boldsymbol{x},y)\in\S, \|\boldsymbol{\delta}\|_p\leq\epsilon \right\}.$$
\end{definition}

In \cite{kalimeris2019sgd}, it was noted that during the initial epochs of training, the model learns a function that is highly correlated with the linear features of the data. We also conduct a simple experiment presented at Remark \ref{rema-linear} in Appendix \ref{def-th-app} to show that the linear separability of dataset will result in the victim model to perform like the linear model.
The following two theorems inform us that when the poison budget exceeds a certain threshold compared to the adversarial budget, the adversarial training set will become linearly separable, in other words, adversarial training will malfunction.

\begin{theorem}
\label{th-def1}
Let $S=\left\{(\boldsymbol{x}_i,y_i)\right\}_{i=1}^N\subset\I^d\times[C]$ be a (clean) training set and $\epsilon\in\R_{>0}$ satisfies $e^{-\frac{d\epsilon^2}{8}}+e^{-\frac{d}{50}}\leq\frac{1}{2NC}$.
Then there exists a class-wise poisoned training set $S^{\poi}$ with poison budget at most $4\epsilon$, such that adversarial training set $\B_{\infty}(S^{\poi},\epsilon)$  is linearly separable.
\end{theorem}

\begin{theorem}
\label{cor}
Let $S=\left\{(\boldsymbol{x}_i,y_i)\right\}_{i=1}^N\subset\I^d\times[C]$ be a (clean) training set. 
If there exists a linearly separable poisoned set of $S$  with poison budget $\epsilon$,
then for any $\eta>0$, there exists a poisoned training set $S^{\poi}$ of $S$ with budget $\eta+\epsilon$, such that the adversarial training set $\B_{\infty}(S^{\poi},\eta)$ is linearly separable by $\F$.
Particularly, $\epsilon=\Omega(\sqrt{\frac{\log NC}{d}})$ satisfies the above condition.
\end{theorem}

The proofs of Theorems \ref{th-def1} and \ref{cor} are deferred to Appendices \ref{proofs-3} and  \ref{proofs-5}.

\begin{table*}[!ht]
\caption{The detection accuracy (\%) for unlearnable examples under Algorithms \ref{alg-weak} and \ref{alg-rule},  the accuracy is for the training and validation set respectively.
\textbf{Linear} represents linear model and \textbf{2-NN} represents two-layer NN.
For Algorithm \ref{alg-rule}, two columns are for bias-shifting noise  $\boldsymbol{\epsilon}_b=\pm 0.5\boldsymbol{e}$.
If dataset is recognized as a clean set, their accuracy are marked in italic.
}
\label{detect-exp22}
\setlength{\tabcolsep}{4pt}
\centering
\begin{tabular}{lcccccccccccccc}
\toprule
\multirow{3}{*}{Poison} & \multicolumn{4}{c}{CIFAR-10}&\multicolumn{4}{c}{CIFAR-100} & \multicolumn{4}{c}{TinyImageNet} \\
& \multicolumn{2}{c}{Algorithm \ref{alg-weak}}  & \multicolumn{2}{c}{Algorithm \ref{alg-rule}}
  & \multicolumn{2}{c}{Algorithm \ref{alg-weak}}
  & \multicolumn{2}{c}{Algorithm \ref{alg-rule}}& \multicolumn{2}{c}{Algorithm \ref{alg-weak}}
  & \multicolumn{2}{c}{Algorithm \ref{alg-rule}}\\
 & \multicolumn{1}{c}{Linear} & \multicolumn{1}{c}{2-NN} & $-0.5$ & $0.5$& \multicolumn{1}{c}{Linear} & \multicolumn{1}{c}{2-NN} & $-0.5$ & $0.5$& \multicolumn{1}{c}{Linear} & \multicolumn{1}{c}{2-NN} & $-0.5$ & $0.5$\\
\midrule
  Clean data & \textit{46.53} &\textit{57.33} &\textit{49.08} & \textit{42.12} & \textit{31.71}&\textit{26.52} & \textit{42.15} & \textit{29.02} & \textit{49.38} & \textit{7.24} &\textit{30.63} & \textit{20.28}\\
  Random(C) &  100.0 &99.13 & 100.0 & 99.84&99.86& 84.19&  99.98 & 97.62&  100.0 &97.82 & 100.0 & 86.02\\
  Region-16 & 99.87& 99.98 & 99.98 & 99.64& 98.39& 88.72& 100.0& 98.76& 99.54&  97.82   & 99.72 & 89.02\\
  Err-min(S) & 99.99&99.77&97.85 & 100.0&99.37& 83.94& 100.0& 100.0& 99.93&97.21&99.99 & 99.27\\
  Err-min(C) &100.0&99.66&100.0 & 100.0&99.96& 84.78&100.0& 100.0& 100.0&  98.29&100.0 & 100.0\\
Err-max & 82.48 & 97.52 & 99.96 & 99.20 & \textit{37.41} & 83.34 & 97.56 & 76.02 & \textit{63.89} & 97.32 & 98.85 & 85.22\\
  RobustEM &78.49 & 99.41 &99.10 & 99.40& \textit{40.25}& 89.94 & 97.84 & 76.00&73.44 & 98.75&94.36 & \textit{68.59}\\
\bottomrule
\end{tabular}
\end{table*}

As the clean dataset has low accuracy on linear model shown in Remark \ref{rem-58}, by Theorems \ref{th-def1} and \ref{cor}, for dataset $S=\left\{(\boldsymbol{x}_i,y_i)\right\}_{i=1}^N\subset\I^d\times[C]$, if the poison budget is more than four times of the adversarial defense budget, or the poison budget is more than a certain number $\epsilon$ of the adversarial defense budget $\eta$, robust unlearnable attacks on $S$ are available.
In other words, adversarial training may fail to defend unlearnable examples of $S$.  This conclusion provides insights into the minimum budget required for adversarial training to effectively defend unlearnable examples.

 The above discussions also explain the results shown in Table \ref{tab-at-cifar10} that poisoned CIFAR-10 with poison budget $8/255$ can not be defended by adversarial training with defense budget $2/255$.
Furthermore, linear separability is achieved for poisoned CIFAR-10 when the gap between poison budget and defense budget reaching $ 8/255$ as shown in Section \ref{section-detection-theo}. 
Therefore, together with Theorem \ref{cor}, our discussions explain why adversarial training with $8/255$ budget cannot defend poisons with $16/255$ budget, as reported in \cite{wang2021fooling}, even though it is only twice as large.

We have established a lower bound for the adversarial budget that can resist unlearnable examples. 
Furthermore,  under some mild assumptions, Theorem \ref{th-upper}  gives a  lower bound for the  poison budget in order to make the poisoned dataset linearly separable, 
whose proof  is provided in Appendix \ref{proofs-6}.

\begin{theorem}
\label{th-upper}
Let the linear model $\F(\boldsymbol{x})=W\boldsymbol{x}+\boldsymbol{b}$, $L(\F(\boldsymbol{x}),y)=\sum_{j\neq y}\max(0, W_j\boldsymbol{x}-W_y\boldsymbol{x}+1)$ be the hinge loss.
Assume that the dataset $S=\{(\boldsymbol{x}_i,y_i)\}_{i=1}^N$ is not linearly separable and  
loss is not less than a constant $\mu_1$,
while $S^{\poi}=\{(\boldsymbol{x}_i+\boldsymbol{\epsilon}_i,y_i)\}_{i=1}^N$ is linearly separable under  $(1, \infty)$-norm regularization and loss is not greater than a constant $\mu_2$.
If  $\mu_1>\mu_2$, then it holds $\max\limits_i \|\boldsymbol{\epsilon}_i\|\geq\frac{\mu_1}{2\mu_2(C-1)}$, where $C$ is the number of classes.
\end{theorem}

Theorem \ref{th-upper} indicates that if the poison  budget is not larger than a certain constant, the poisoned training set will not be linearly separable, which results in a failure of robust unlearnable attacks. 
Table \ref{tab-diff-poison} demonstrates that the poison budget $2/255$ is not enough to make Region-16 and Err-min(S) poisoned CIFAR-10 linearly separable. 
Together with Theorem \ref{cor},
our discussions explain why adversarial training with budget $6/255$ is effective in defending against poisons with budget $8/255$ as shown in Table \ref{tab-at-cifar10}.

\section{Experimental results}
\label{sec-exp11}
In this section, experimental results are provided to verify the detection algorithms in Sections \ref{section-detection-simp} and \ref{section-detection-bs}, defense power of stronger data augmentations with adversarial noises in Section \ref{strong-aug}, and the theoretical criteria of adversarial training in Section \ref{def-ats}.
Experimental setups  are given in  Appendix \ref{sec-learn-detail}.
Our codes are available at https://github.com/hala64/udp.

\subsection{Experimental results on poison detection}
\label{sec-experiment1}

Experimental results for  Algorithms \ref{alg-weak} and  \ref{alg-rule}  are given in Table \ref{detect-exp22}. The detection bound  $B$ is $0.7$. For more results on poison detection, please refer to Appendix \ref{sec-7.4}.

\paragraph{Detection performance.}
Results in Table \ref{detect-exp22} show that most of the poisons can be detected by linear models and bias-shifting noise tests with $\boldsymbol{\epsilon}_b=0.5\boldsymbol{e}$, and all of them can be detected by two-layer NN and bias-shifting noise with $\boldsymbol{\epsilon}_b=-0.5\boldsymbol{e}$.
Combined with results in Appendix \ref{sec-7.4}, even for robust unlearnable examples, such as RobustEM and Adv Inducing \cite{wang2021fooling}, or for poisons designed to reduce robust generalization power, such as Hypocritical \cite{tao2022can}, our detection methods still work.

\paragraph{Different poison ratios.}
Although unlearnable examples only work when all training data is poisoned, our detection methods still perform well even only a part of data is poisoned. 
Table \ref{detect-ratios-2} presents detection results for different poison ratios. 
It is shown that as long as $60\%$ of data are unlearnably poisoned, it will be detected by the victims.

\begin{table}[htbp]
    \caption{Detection accuracy  under different poison ratios by Algorithm \ref{alg-rule} with $\boldsymbol{\epsilon}_b=-0.5\boldsymbol{e}$.}
    \label{detect-ratios-2}
    \centering
\begin{tabular}{lcccccc}
\toprule
Ratio
 & 100\% & 80\% & 60\% & 40\% & 20\%\\
\midrule
 Clean data  &\textit{49.08}&\textit{---}&\textit{---}&\textit{---}&\textit{---}\\
 Random(C)  & 100.0  & 86.24 & 77.53 & \textit{65.92} & \textit{53.07}\\
 Region-16 & 99.98 & 87.15 & 78.77 & 70.53 &\textit{57.90}\\
 Err-min(S) & 97.85 & 85.06 & 73.66 & 71.59 & 70.73\\
 Err-min(C) & 100.0 & 85.38 & 75.18 & \textit{62.07} & \textit{53.44}\\
 RobustEM & 99.10 & 86.71 & 73.06 & \textit{60.92} & \textit{51.52} \\
\bottomrule
\end{tabular}

\end{table}

\begin{table*}[!htbp]
\caption{Test accuracy (\%) of different defense methods for poisoned CIFAR-10.}
\label{tab-aug-def1}
\centering
\setlength{\tabcolsep}{4.5pt}
\begin{tabular}{lcccccccccc}
 \toprule
Method/Poison & Region-16 & Err-min(S) & Err-max & RobustEM & NTGA & AR & EntF  & TUE \\
  \midrule
  \setlength{\tabcolsep}{5.5pt}
  No defense & 19.86 & 10.09 & 7.19 & 25.30 & 11.23 & 17.18 & 83.10 & 10.00\\
  AT-based methods & 85.78 & 84.00 & \textbf{84.75} & 81.06 & 84.19 & 85.25 & 73.61 & 83.94\\
  Adversarial Noises & 56.35 & 10.72  & 28.74 & 25.69 & 17.96 & 14.57 & 48.15 & 90.07\\
  UEraser & 82.66 & 86.37 & 47.46 & 78.39 & 82.15 & \textbf{87.21} & 87.40 & 75.96 \\
  ISS & 67.11 & 85.24 & 84.36 & \textbf{83.84} & 85.64 & 85.11 & 77.02 & 84.32\\
  SDA(ours) &  77.20 & 75.62  & 56.96 & 79.21 & 78.48 & 75.81 & \textbf{90.84}  & 73.80 \\
  AN+SDA(ours) &  \textbf{93.51} & \textbf{88.01}  & 61.19 & 79.28 & \textbf{89.00} & 80.20 & 88.17 & \textbf{92.76}\\
  \bottomrule
\end{tabular}
\end{table*}

\begin{table*}[!ht]
\caption{Test accuracy (\%) when conducting adversarial training with budgets $i/255,i=0,\ldots,16$ to defend poisoned CIFAR-10 with the  poison budget $\epsilon=8/255$. We do not use any data augmentation here for better verification of our theorems.}
\label{tab-at-cifar10}
\centering
\setlength{\tabcolsep}{4.5pt}
\begin{tabular}{cccccccccccc}
 \toprule
 Poison & 0 & 1/255 & 2/255 &3/255& 4/255 & 6/255 & 8/255 & 12/255 & 16/255\\
\midrule
   Region-16&19.86 & 24.32 & 29.57 &50.09& 72.13&\textbf{77.03}&72.65 &67.65&62.78\\
  Err-min(S)&10.09&10.01&10.13&18.64&69.92&\textbf{76.53} &72.13&67.23&61.70\\
  RobustEM&25.30&24.92&28.50&33.74&46.01&\textbf{76.69} &72.19&63.16&53.34\\
\bottomrule
\end{tabular}
\end{table*}

\paragraph{False Poisitives and False Negatives}
It is worth noting that our detection methods are designed to assess entire datasets, rather than individual data instances. Therefore, we analyze the false positives and false on different unlearnable datasets under varying detection bound $B$.
With reference to Tables \ref{detect-exp22} and \ref{detect1-cifar10}, we can evaluate the occurrences of FP and FN across a range of clean and poisoned datasets. The results are presented in Table \ref{fp-fn}.
 \begin{table}[H]
 \caption{False positives and false negatives of our detection methods across different datasets and unlearnable examples.}
 \label{fp-fn}
\centering
\setlength{\tabcolsep}{2.5pt}
\begin{tabular}{lccccccccc}
 \toprule
Bound/(FP/FN) & 0.1 & 0.2 & 0.3 & 0.4 & 0.5 & 0.6 & 0.7 & 0.8 & 0.9 \\
  \midrule
  Bias+0.5 & 3/0 & 3/0 & 1/1 & 1/1 & 0/1 & 0/1 & 0/2 & 0/5 & 0/8 \\
  Bias-0.5 & 3/0 & 3/1 & 3/1 & 2/1 & 0/1 & 0/1 & 0/1 & 0/1 & 0/3 \\
  Linear & 3/0 & 3/0 & 3/0 & 2/1 & 0/4 & 0/5 & 0/6 & 0/10 & 0/13 \\
  2-NN & 2/0 & 2/0 & 1/0 & 1/0 & 1/0 & 0/1 & 0/1 & 0/2 & 0/9 \\
  \bottomrule
\end{tabular}
\end{table}

\subsection{Experimental results on poison defense}
\label{exp-defense}
We evaluate the defense power on Algorithm \ref{alg-strong-aug} with AT-based methods and recently proposed defense method UEraser \cite{qin2023learning} and ISS \cite{liu2023image}.
%
%
The results demonstrate that even without adversarial noises, using stronger data augmentation alone can achieve comparable defense power to adversarial training, which indicates that breaking detectability  is a key to defending against unlearnable examples.

Furthermore, our defense method achieves state-of-the-art performance on most of the existing unlearnable examples, such as Region-16, Err-min(S), NTGA and TUE poisons, and achieve the comparable defense power for RobustEM and AR poisons, only perform a little suboptimally for Err-max poisons, as shown in Table \ref{tab-aug-def1} compared to Tables \ref{tab-at-cifar10} and \ref{tab-adv-def}. 
We also evaluate on recently proposed poison method deceiving adversarial training, called EntF \cite{wen2023adversarial}, and results also show advantage of our method. 
Additional experimental results and more discussions are provided in Appendix \ref{strong-aug-app}.

Moreover, since we only conduct adversarial training on the simple two-layer NN, our method is much more time-efficient than AT-based methods, {as shown in Table \ref{tab-noise-time}, adversarial noises generated on two-layer NN is more than 3 times faster than generated on ResNet18.}

\subsection{Evaluation of criteria between the poison budget and the defense budget}
\label{sec-experiment2}

Table \ref{tab-at-cifar10} shows the test accuracy under different adversarial defense budgets.
For adversarial training with budget $\epsilon\leq2/255$ and unlearnable poisoning attacks with budget $\eta=8/255\geq4\epsilon$, the test accuracy is less than $30\%$, which is much lower than the linear separability rate $46.53\%$ of CIFAR-10.
This verifies Theorem \ref{th-def1}  that adversarial training fails when defense budget is too small.

Additionally, with the increase of defense budgets, the defense power of adversarial training initially increases rapidly, but then begins to gradually decrease.
This is because when the defense budget increases, the gap between poison and defense budgets decreases, eventually becoming too small to achieve linear separability. As shown in Table \ref{tab-at-cifar10}, adversarial training is effective when defense budget reaches to $6/255$, since the gap $2/255$ is not large enough to make the poisoned dataset linearly separable, which verifies Theorems \ref{cor} and \ref{th-upper}.

It is worth noting that different from \citep{tao2021better}, adversarial training here is the tool to maintain accuracy rather than robustness. Therefore, the effective defense budget here is smaller than \citep{tao2021better}.
From Table \ref{tab-diff-poison}, RobustEM never becomes  linearly separable, indicating that achieve linear separability is a sufficient but not necessary condition for generating unlearnable examples.
With further increase in the defense budget, the trade-off between accuracy and robustness emerges \cite{tsipras2018robustness}, leading to a gradual drop in test accuracy. 
For more discussions and experiments on this topic, please refer to Appendix \ref{poison-and-def}.

\begin{table}[H]
\caption{Linear separability rate of poisoned CIFAR-10 with different poison budget $\eta$.}
\label{tab-diff-poison}
\centering
\setlength{\tabcolsep}{3.5pt}
\begin{tabular}{cccccccc}
 \toprule
  Poison/budget & 0 & 1/255 & 2/255 & 4/255 & 6/255 & 8/255\\
  \midrule
  Region-16 & 46.53 & 67.85 & 89.06 & 98.32 & 99.65& 99.87\\
  Err-min(S) & 46.53 & 76.63 & 93.15 & 99.87 & 99.97& 99.99\\
  RobustEM & 46.53 & 55.44 & 61.29 & 69.03 & 74.62 & 78.49\\
  \bottomrule
\end{tabular}
\end{table}

\section{Conclusion}
\label{sec-conc}

In this paper, we demonstrate that unlearnable examples can be easily detected.
We prove that linear separability always exists for certain unlearnable poisoned dataset, and propose effective detection methods. 
We use stronger data augmentations with adversarial noises of simple networks to achieve effective defense for unlearnable examples.
Furthermore, we derive a certified upper bound for poison budget relative to the adversarial budget on adversarial training.

\paragraph{Limitations and future work.}
%
From Table \ref{tab-at-cifar10}, the results in Theorem \ref{th-def1} have rooms for improvements, that is, tighter upper bounds for poison budget are possible.
Certified defense methods against unlearnable examples need to be developed in the future. 
It is desirable to craft more potent defense methods against Err-max poisons.
Moreover, in the pursuit of privacy preservation, it is imperative to create more sophisticated unlearnable examples that resist existing detection and defense methods.

\textbf{Acknowledgement}
This work is supported by NKRDP grant No.2018YFA0306702 and NSFC grant No.11688101.

\bibliography{ref}

\appendix
\onecolumn


\section{Proofs}
\label{proofs}

\subsection{Proof of Theorem \ref{th-ex1}}
\label{proofs-1}
\begin{lemma}[McDiarmid's Inequality \cite{mcdiarmid1989method}]
Let $X_1,X_2,\cdots,X_n$ be independent random variables on $\mathcal{X}_1,\mathcal{X}_2,\cdots,\mathcal{X}_n$ and $f:\mathcal{X}_1\times\mathcal{X}_2\times\cdots\times\mathcal{X}_n\to\mathbb{R}$ be a multivariate function. 
If there exists positive constants $c_1,c_2,\cdots,c_n$, such that for all $(\boldsymbol{x}_1,\boldsymbol{x}_2,\cdots,\boldsymbol{x}_n)\in\mathcal{X}_1\times\mathcal{X}_2\times\cdots\times\mathcal{X}_n$ and $i\in[n]$, 
it has
$$\sup\limits_{\boldsymbol{x}_i^{\prime}\in\mathcal{X}_i}|f(\boldsymbol{x}_1,\cdots,\boldsymbol{x}_{i-1},\boldsymbol{x}_i^{\prime},\boldsymbol{x}_{i+1},\cdots,\boldsymbol{x}_n)-f(\boldsymbol{x}_1,\cdots,\boldsymbol{x}_{i-1},\boldsymbol{x}_i,\boldsymbol{x}_{i+1},\cdots,\boldsymbol{x}_n)|\leq c_i,$$ then for any $\epsilon>0$, the following inequalities hold
\begin{equation*}
\begin{array}{l}
\Prob(f(X_1,X_2,\cdots,X_n)-\mathbb{E}\left[f(X_1,X_2,\cdots,X_n)\right]\geq\epsilon)\leq e^{-\frac{2\epsilon^2}{\sum_{i=1}^n c_i^2}},\\
\Prob(f(X_1,X_2,\cdots,X_n)-\mathbb{E}\left[f(X_1,X_2,\cdots,X_n)\right]\leq-\epsilon)\leq e^{-\frac{2\epsilon^2}{\sum_{i=1}^n c_i^2}}.\\
\end{array}
\end{equation*}
\end{lemma}

\begin{proof}[Proof of Theorem \ref{th-ex1}]
By McDiarmid's inequality, $\forall \gamma>0$ and $i\neq j$, it holds that
$$\textup{Prob}(\left<\boldsymbol{v}_i,\boldsymbol{v}_j\right>\geq\gamma)\leq e^{-\frac{\gamma^2}{2d\epsilon^4}},$$
and $\forall i,j$ it holds that
$$\textup{Prob}\left(\left|\left<\boldsymbol{x}_i,\boldsymbol{v}_j\right>\right|\geq\gamma\right)\leq 2e^{-\frac{\gamma^2}{2d\epsilon^2}}.$$
Let $\gamma=\frac{d\epsilon^2}{4}$, then it has
$$\textup{Prob}\left(\left<\boldsymbol{v}_i,\boldsymbol{v}_j\right>\geq\frac{d\epsilon^2}{4}\right)\leq e^{-\frac{d}{32}}, i\neq j$$
and
$$\textup{Prob}\left(\left|\left<\boldsymbol{x}_i,\boldsymbol{v}_j\right>\right|\geq\frac{d\epsilon^2}{3}\right)\leq 2e^{-\frac{d\epsilon^2}{18}}.$$
Denote the linear classifier $\F(\boldsymbol{x})=W_{\boldsymbol{v}}\boldsymbol{x}$, where
$$W_{\boldsymbol{v}}=[\boldsymbol{v}_1,\cdots,\boldsymbol{v}_C]^T\in\R^{C\times d}.$$
Then given $(\boldsymbol{x}+\boldsymbol{v}_{y},y)\in D^{\poi}$, it has
$$\F\left(\boldsymbol{x}+\boldsymbol{v}_{y}\right)=W_{\boldsymbol{v}}(\boldsymbol{x}+\boldsymbol{v}_{y})=\left(\left<\boldsymbol{v}_1,\boldsymbol{x}+\boldsymbol{v}_{y}\right>,\cdots,\left<\boldsymbol{v}_C,\boldsymbol{x}+\boldsymbol{v}_{y}\right>\right)^T.$$
Then if $\forall k\in\left[C\right]$, $|\left<\boldsymbol{v}_k,\boldsymbol{x}\right>|<\frac{d\epsilon^2}{3}$ and $\forall k\in\left[C\right]\backslash\{y\},\left<\boldsymbol{v}_k,\boldsymbol{v}_y\right><\frac{d\epsilon^2}{4}$, it holds that
$$\left<\boldsymbol{v}_y,\boldsymbol{x}+\boldsymbol{v}_y\right>>\frac{2d\epsilon^2}{3}>\frac{7d\epsilon^2}{12}>\left<\boldsymbol{v}_y,\boldsymbol{x}+\boldsymbol{v}_k\right>, \forall k\in\left[C\right]\backslash\{y\},$$
which directly leads to correctness of classifier $\F$.
Therefore, it follows that
\begin{equation*}
\begin{array}{ll}
&\textup{Prob}\left[\arg\max\limits_j\F(\boldsymbol{x}+\boldsymbol{v}_y)_j=y\right] \vspace{1ex} \\
\geq&\textup{Prob}\left[|\left<\boldsymbol{v}_k,\boldsymbol{x}\right>|<\frac{d\epsilon^2}{4},\forall k\in[C] \cap \left<\boldsymbol{v}_k,\boldsymbol{v}_y\right><\frac{d\epsilon^2}{4},\forall k\in\left[C\right]\backslash\{y\}\right] \vspace{1ex} \\
\geq&1-\left(1-\left(1-2e^{-\frac{d\epsilon^2}{18}}\right)^C\right)-\left(1-\left(1-e^{-\frac{d}{32}}\right)^{C-1}\right) \vspace{1ex} \\
\geq&1-2Ce^{-\frac{d\epsilon^2}{18}}-(C-1)e^{-\frac{d}{32}}.
\end{array}
\end{equation*}
Then the probability of $\arg\max\limits_j\F(\boldsymbol{x}_i+\boldsymbol{v}_{y_i})_j=y_i, \forall i\in\left[N\right]$ is at least
\begin{equation*}
\begin{array}{ll}
&1-N\left(1-\textup{Prob}\left[\arg\max\limits_j\F(\boldsymbol{x}+\boldsymbol{v}_y)_j=y\right]\right) \vspace{1ex} \\
\geq & 1-NC\left(2e^{-\frac{d\epsilon^2}{18}}+e^{-\frac{d}{32}}\right),
\end{array}
\end{equation*}
which complete the proof.
\end{proof}

\subsection{Proof of Theorem \ref{th-ex1-bc}}
\label{proofs-2}
\begin{proof}
The proof is similar to the proof of Theorem \ref{th-ex1}. \\
By McDiarmid's inequality, $\forall \gamma>0$ and $i\neq j$, it holds that
$$\textup{Prob}(\left<\boldsymbol{v}_i,\boldsymbol{v}_j\right>\geq\gamma)\leq e^{-\frac{\gamma^2k}{2d^2\epsilon^4}},$$
and $\forall i,j$ it holds that
$$\textup{Prob}\left(\left|\left<\boldsymbol{x}_i,\boldsymbol{v}_j\right>\right|\geq\gamma\right)\leq 2e^{-\frac{\gamma^2k}{2d^2\epsilon^2}}.$$
The remaining details are the same as Theorem \ref{th-ex1}.
\end{proof}

\subsection{Proof of Theorem \ref{th-def1}}
\label{proofs-3}
\begin{proof}
For class-wise noise vectors $\left\{\boldsymbol{v}_i\right\}_{i=1}^C$ satisfy that $\forall i\in\left[C\right]$ and $j\in\left[d\right]$, $(\boldsymbol{v}_i)_j$ is i.i.d. and obey distribution $\Delta(4\epsilon)$, where
$$\Delta(\epsilon)=2\epsilon\cdot \text{Bernoulli}\left(\frac{1}{2}\right)-\epsilon,$$
that is,
 $(\boldsymbol{v}_i)_j$ with $\frac{1}{2}$ probability equals $4\epsilon$ and $-4\epsilon$ respectively.

Denote the poisoned set as $S^{\poi}=\left\{(\boldsymbol{x}_i+\boldsymbol{v}_{y_i},y_i)\right\}_{i=1}^N$.
Following Theorem \ref{th-ex1}, with probability at least $1-NC\left(2e^{-\frac{d\epsilon^2}{8}}+e^{-\frac{d}{32}}\right)$,
the linear model $\boldsymbol{V}=[\boldsymbol{v}_1,\cdots,\boldsymbol{v}_C]^T$ can classify the data set $S'=\left\{(\boldsymbol{x}_i+\frac{\boldsymbol{v}_{y_i}}{2},y_i)\right\}_{i=1}^N$ correctly.

By McDiarmid's inequality, $\forall i\neq j$, it holds that
$$\textup{Prob}\left(\left<\boldsymbol{v}_i,\boldsymbol{v}_j\right>\geq\frac{16d\epsilon^2}{5}\right)\leq e^{-\frac{d}{50}}.$$
Then such random vector $\left\{\boldsymbol{v}_i\right\}_{i=1}^C$ holds with probability at least $1-\frac{C(C-1)}{2} e^{-\frac{d}{50}}$.
Therefore, with probability at least
$$1-NC\left(2e^{-\frac{d\epsilon^2}{8}}+e^{-\frac{d}{32}}\right)-\frac{C(C-1)}{2} e^{-\frac{d}{50}}>1-NC\left(2e^{-\frac{d\epsilon^2}{8}}+2e^{-\frac{d}{50}}\right)>0$$
that $\left\{\boldsymbol{v}_i\right\}_{i=1}^C$ satisfies all conditions above, which implies such $\left\{\boldsymbol{v}_i\right\}_{i=1}^C$ exists in $\B_{\infty}(4\epsilon)$ ball.

Now we choose such $\left\{\boldsymbol{v}_i\right\}_{i=1}^C$ as the injected noise to generate the poisoned data set $S^{\poi}$. Then for all perturbation $||\boldsymbol{\delta}_i||_{\infty}\leq\epsilon$ and the data set $\widetilde{S}=\left\{(\boldsymbol{x}_i+\boldsymbol{v}_{y_i}+\boldsymbol{\delta}_i,y_i)\right\}_{i=1}^N \subset B_{\infty}(S^{\poi},\epsilon)$, it holds that
\begin{equation*}
\renewcommand{\arraystretch}{1.3}
\begin{array}{ll}
&\left<\boldsymbol{x}_i+\boldsymbol{v}_{y_i}+\boldsymbol{\delta}_i,\boldsymbol{v}_{y_i}\right>- \left<\boldsymbol{x}_i+\boldsymbol{v}_{y_i}+\boldsymbol{\delta}_i,\boldsymbol{v}_k\right> \vspace{1ex} \\
 =&\left<\boldsymbol{x}_i+\frac{\boldsymbol{v}_{y_i}}{2},\boldsymbol{v}_{y_i}\right>- \left<\boldsymbol{x}_i+\frac{\boldsymbol{v}_{y_i}}{2},\boldsymbol{v}_k\right> + \left<\frac{\boldsymbol{v}_{y_i}}{2}+\boldsymbol{\delta}_i,\boldsymbol{v}_{y_i}\right>- \left<\frac{\boldsymbol{v}_{y_i}}{2}+\boldsymbol{\delta}_i,\boldsymbol{v}_k\right> \vspace{1ex} \\
 =&\left<\boldsymbol{x}_i+\frac{\boldsymbol{v}_{y_i}}{2},\boldsymbol{v}_{y_i}\right>- \left<\boldsymbol{x}_i+\frac{\boldsymbol{v}_{y_i}}{2},\boldsymbol{v}_k\right> +\left<\boldsymbol{\delta}_i,\boldsymbol{v}_{y_i}-\boldsymbol{v}_k\right> +8d\epsilon^2-\frac{1}{2}\left<\boldsymbol{v}_{y_i},\boldsymbol{v}_k\right> \vspace{1ex} \\
 \geq&\left<\boldsymbol{\delta}_i,\boldsymbol{v}_{y_i}-\boldsymbol{v}_k\right>+\frac{32}{5}d\epsilon^2 \vspace{1ex} \\
 \geq&\frac{32}{5}d\epsilon^2-\|\boldsymbol{\delta}_i\|_{\infty}\|\boldsymbol{v}_{y_i}-\boldsymbol{v}_k\|_1 \vspace{1ex} \\
 \geq&\frac{32}{5}d\epsilon^2-\epsilon\sqrt{d}\|\boldsymbol{v}_{y_i}-\boldsymbol{v}_k\|_2 \vspace{1ex} \\
 =&\frac{32}{5}d\epsilon^2-\epsilon\sqrt{d}\sqrt{32d\epsilon^2+\left<\boldsymbol{v}_{y_i},\boldsymbol{v}_k\right>} \vspace{1ex} \\
 \geq&\frac{32}{5}d\epsilon^2-\sqrt{32+\frac{16}{5}}d\epsilon^2 \vspace{1ex}  \\
 >&0
\end{array}
\end{equation*}
for all $k\neq y_i$ and $i=[N]$.
Therefore, for all dataset $\widetilde{S} \subset \B_{\infty}(S^{\poi},\epsilon)$, the linear model $V$ can classify them correctly.
\end{proof}

\subsection{Proof of Theorem \ref{cor}}
\label{proofs-5}

\begin{lemma}
\label{th-append}
Let $S=\left\{(\boldsymbol{x}_i,y_i)\right\}_{i=1}^N\subset\I^d\times[C]$ be a (clean) training set and
 $\boldsymbol{v}_i\in\{-\epsilon,\epsilon\}^d$,  $i\in[C]$
is the class-wise poison, 
such that the linear network $\F(\boldsymbol{x})=[\boldsymbol{v}_1,\cdots,\boldsymbol{v}_C]^T\boldsymbol{x}$ can correctly classify the poisoned training set $\{\boldsymbol{x}_i+\boldsymbol{v}_{y_i},y_i\}_{i=1}^N$. Then for any $\eta\in\R_{>0}$, there exists a poisoned training set $S^{\poi}$ of $S$ with poison budget $\eta+\epsilon$, 
such that  $\B_{\infty}(S^{\poi},\eta)$  is linearly separable by $\F$.
\begin{proof}
We will show that  $S^{\poi}=\{\boldsymbol{x}_i+\boldsymbol{v}_{y_i}+\eta\cdot \sign(\boldsymbol{v}_{y_i}),y_i\}_{i=1}^N$ is a poisoned set verifying the theorem.
For all adversarial perturbation $\|\boldsymbol{\delta}\|_{\infty}\leq\eta$, 
$\forall \left(\boldsymbol{x}+\boldsymbol{v}_y+\eta\cdot \sign(\boldsymbol{v}_y),y\right)\in S^{\poi}$, when $i\neq y$, it holds that
\begin{equation*}
\renewcommand{\arraystretch}{1.3}
\begin{array}{ll}
&\F(\boldsymbol{x}+\boldsymbol{v}_y+\eta\cdot \sign(\boldsymbol{v}_y)+\boldsymbol{\delta})_y-\F(\boldsymbol{x}+\boldsymbol{v}_y+\eta\cdot \sign(\boldsymbol{v}_y)+\boldsymbol{\delta})_i \vspace{1ex} \\
=&\left<\boldsymbol{x}+\boldsymbol{v}_y+\eta\cdot \sign(\boldsymbol{v}_y)+\boldsymbol{\delta},\boldsymbol{v}_y-\boldsymbol{v}_i\right> \vspace{1ex} \\
> &\left<\eta\cdot \sign(\boldsymbol{v}_y)+\boldsymbol{\delta},\boldsymbol{v}_y-\boldsymbol{v}_i\right> (\textup{By the correct classification of}\, (\boldsymbol{x}+\boldsymbol{v}_y,y)) \vspace{1ex} \\
=& \eta\epsilon d +\epsilon \left<\boldsymbol{\delta},\sign(\boldsymbol{v}_y)-\sign(\boldsymbol{v}_i)\right>-\eta\epsilon\left<\sign(\boldsymbol{v}_y),\sign(\boldsymbol{v}_i)\right> \vspace{1ex} \\
\geq&  \eta\epsilon d - \eta\epsilon\|\sign(\boldsymbol{v}_y)-\sign(\boldsymbol{v}_i)\|_1 - \eta\epsilon \left<\sign(\boldsymbol{v}_y),\sign(\boldsymbol{v}_i)\right> \vspace{1ex} \\
= & \eta\epsilon (d - \|\sign(\boldsymbol{v}_y)-\sign(\boldsymbol{v}_i)\|_1 - \left<\sign(\boldsymbol{v}_y),\sign(\boldsymbol{v}_i)\right>).
\end{array}
\end{equation*}
Since $\left<\boldsymbol{a},\boldsymbol{b}\right>+\|\boldsymbol{a}-\boldsymbol{b}\|_1=d$ for $\boldsymbol{a},\boldsymbol{b}\in\{-1,+1\}^d$, it holds that
$$d - \|\sign(\boldsymbol{v}_y)-\sign(\boldsymbol{v}_i)\|_1 -  \left<\sign(\boldsymbol{v}_y),\sign(\boldsymbol{v}_i)\right> =0.$$
Therefore, $\F(\boldsymbol{x}+\boldsymbol{v}_y+\eta\cdot \sign(\boldsymbol{v}_y)+\boldsymbol{\delta})_y-\F(\boldsymbol{x}+\boldsymbol{v}_y+\eta\cdot \sign(\boldsymbol{v}_y)+\boldsymbol{\delta})_i>0$, that is, the linear network $\F(\boldsymbol{x})=[\boldsymbol{v}_1,\cdots,\boldsymbol{v}_C]^T\boldsymbol{x}$ can correctly classify $\B(S^{\poi},\epsilon)$.
It is clear that $S^{\poi}$ has poison budget $ \eta+\epsilon$.
The theorem is proved.
\end{proof}
\end{lemma}

\begin{proof}[Proof of Theorem \ref{cor}]
By Theorem \ref{th-ex1}, if the gap $\epsilon$ satisfies that $1-NC\left(2e^{-\frac{d\epsilon^2}{18}}+e^{-\frac{d}{32}}\right)>0$, such linear network $\F$ in Lemma \ref{th-append} exists.
Therefore, the gap $\epsilon$ satisfies 
$$\left(2e^{-\frac{d\epsilon^2}{18}}+e^{-\frac{d}{32}}\right)<\frac{1}{NC}.$$
As images lie in $[0,1]$, it holds that $\epsilon\leq1$. 
Therefore, it holds that
$\left(2e^{-\frac{d\epsilon^2}{18}}+e^{-\frac{d}{32}}\right)\leq 3e^{-\frac{d\epsilon^2}{32}}$. Then as long as $\epsilon$ satisfies 
$$3e^{-\frac{d\epsilon^2}{32}}<\frac{1}{NC},$$ that is, $\epsilon=\Omega(\sqrt{\frac{\log NC}{d}})$, the condition of Theorem \ref{th-ex1} is satisfied. The proof is completed.
\end{proof}

\vskip10pt
\subsection{Proof of Theorem \ref{th-upper}}
\label{proofs-6}

\begin{theorem}[restated]
\label{th-upper-re}
Let the linear model $\F(\boldsymbol{x})=W\boldsymbol{x}+\boldsymbol{b}$, $L(\F(\boldsymbol{x}),y)=\sum_{j\neq y}\max(0, W_j\boldsymbol{x}-W_y\boldsymbol{x}+1)$ be the hinge loss.
Assume that the dataset $S=\{(\boldsymbol{x}_i,y_i)\}_{i=1}^N$ is not linearly separable and  
$\frac{1}{N}\sum\limits_{i=1}^N L(\F(\boldsymbol{x}_i),y_i) \geq \mu_1$ 
for a constant $\mu_1$,
while $S^{\poi}=\{(\boldsymbol{x}_i+\boldsymbol{\epsilon}_i,y_i)\}_{i=1}^N$ is linearly separable under  $(1, \infty)$-norm regularization and $\frac{1}{N}\sum\limits_{i=1}^N L(\F(\boldsymbol{x}_i+\boldsymbol{\epsilon}_i),y_i)+\|W\|_{1,\infty}\leq \mu_2$ for a constant $\mu_2$.
If  $\mu_1>\mu_2$, then it holds $\max\limits_i \|\boldsymbol{\epsilon}_i\|\geq\frac{\mu_1}{2\mu_2(C-1)}$, where $C$ is the number of classes.
\end{theorem}
\begin{proof}
Let 
$K=\frac{1}{N}\sum\limits_{i=1}^N L(\F(\boldsymbol{x}_i+\boldsymbol{\epsilon}_i),y_i)=\frac{1}{N}\sum\limits_{i=1}^N\sum\limits_{j\neq y_i}\max(0, (W_j-W_{y_i})(\boldsymbol{x}_i+\boldsymbol{\epsilon}_i)+1).$ 

Then it has 
$\|W\|_{1,\infty}\leq \mu_2-K.$
As 
$\frac{1}{N}\sum\limits_{i=1}^N L(\F(\boldsymbol{x}_i),y_i) \geq \mu_1,$ 
there exists $i$ such that
$$\sum\limits_{j\neq y_i}\max(0, (W_j-W_{y_i})\boldsymbol{x}_i+1)\geq \mu_1.$$
Therefore, it has 
$$\sum\limits_{j\neq y_i}\left[ \max(0, (W_j-W_{y_i})(\boldsymbol{x}_i+\boldsymbol{\epsilon}_i)+1)-\max(0, (W_j-W_{y_i})\boldsymbol{x}_i+1) \right] \leq K-\mu_1,$$
then there exists $j$ such that
$$\max(0, (W_j-W_{y_i})(\boldsymbol{x}_i+\boldsymbol{\epsilon}_i)+1)-\max(0, (W_j-W_{y_i})\boldsymbol{x}_i+1)\leq \frac{K-\mu_1}{C-1}.$$
As $K-\mu_1<K-\mu_2\leq \|W\|_{1,\infty}\leq 0$, it holds that
$$|(W_j-W_{y_i})\boldsymbol{\epsilon}_i|\geq \frac{\mu_1-K}{C-1}.$$
Therefore, it has 
$$\|W_j-W_{y_i}\|_1\|\boldsymbol{\epsilon}_i\|_{\infty}\geq\frac{\mu_1-K}{C-1},$$
which conclude that
$$\|\boldsymbol{\epsilon}_i\|_{\infty}\geq\frac{\mu_1-K}{2(\mu_2-K)(C-1)}\geq\frac{\mu_1}{2\mu_2(C-1)}$$
as $\mu_1>\mu_2$.
\end{proof}

\section{Experimental setup and details of learning schedules}
\label{sec-learn-detail}

In Section \ref{section-detection}, when Algorithm \ref{alg-weak} is used, we set the initial learning rate to be 0.01; and when Algorithm \ref{alg-rule} is used, we set the initial learning rate to be 0.1. 
The weight decays and momentums are set as the default setting.
For Algorithm \ref{alg-rule}, we do not use any data augmentations.
In Section \ref{strong-aug} and Appendix \ref{strong-aug-app},  we use learning schedule with 200 total epochs with cosine learning rates annealing.

In Section \ref{def-ats} and Appendix \ref{sec-7.5}, vallina AT means that we use PGD-10 as the adversarial perturbation.
On evaluation of unlearnable examples by adversarial training, for CIFAR-10, we use learning schedule with 110 total epochs and learning rate decays by 0.1 at the 75-th, 90-th and 100-th epoch; 
and for CIFAR-100, we use learning schedule with 200 total epochs and learning rate decays by 0.1 at the 80-th, 140-th and 180-th epoch. 
The initial learning rate is set to be 0.1 for all of three datasets.

In Appendix \ref{sec-7.1}, for CIFAR-10, we use the learning schedule  with 100 total epochs and learning rate decays by 0.1 at the 75-th and 90-th epoch;
and for CIFAR-100 and TinyImageNet, we use the learning schedule with 200 total epochs and learning rate decays by 0.1 at the 60-th, 120-th and 160-th epoch.
Our initial learning rate is set to 0.1, and when evaluating the highest test accuracy the poisoned training set could achieve, we would try smaller learning rates 0.01, 0.001 and 0.0001.

Throughout of our experiments, we use SGD as the optimizer with momentum to be 0.9 and weight decay to be $5\times10^{-4}$, and without special annotations, we use ordinary data augmentations random crop and random horizontal flip for CIFAR-10 and CIFAR-100; 
and we use random crop, random horizontal flip, color jitter 
\cite{shijie2017research}
, and cutout 
\cite{devries2017improved}
for TinyImageNet.
We conduct all of our experiments on a single NVIDIA 3090 GPU.

\section{Poisoned data and noises}
\label{poison-img}

In this section, we give several types of poisoned data and its corresponding perturbations for illustrations.

\begin{minipage}[ht]{0.45\textwidth}
\begin{figure}[H]
\centering
\hspace{2mm}
\includegraphics[width=6.0truecm]{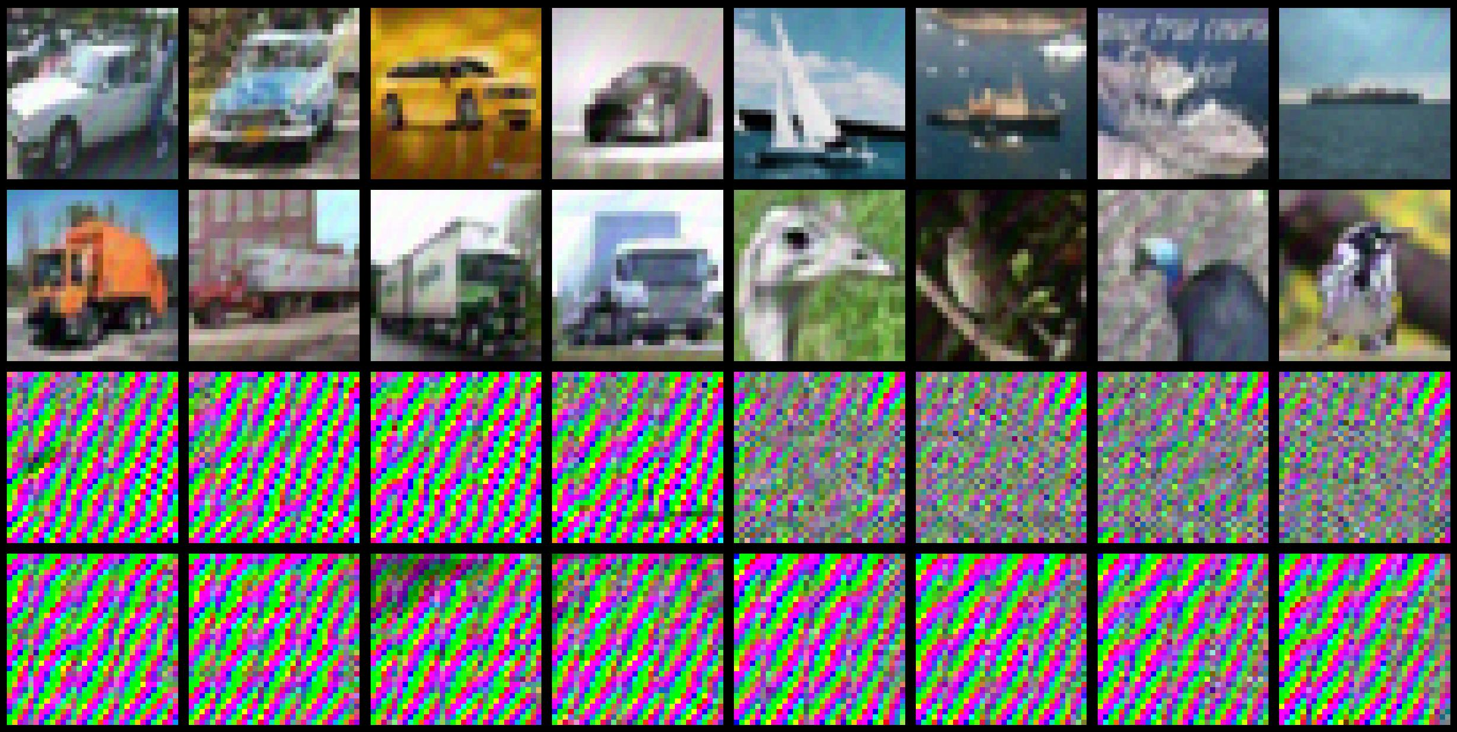}
\caption{Err-min(S) poisoned CIFAR-10 and its corresponding perturbations.}
\end{figure}
\end{minipage}
\quad
\begin{minipage}[ht]{0.45\textwidth}
\begin{figure}[H]
\centering
\hspace{2mm}
\includegraphics[width=6.0truecm]{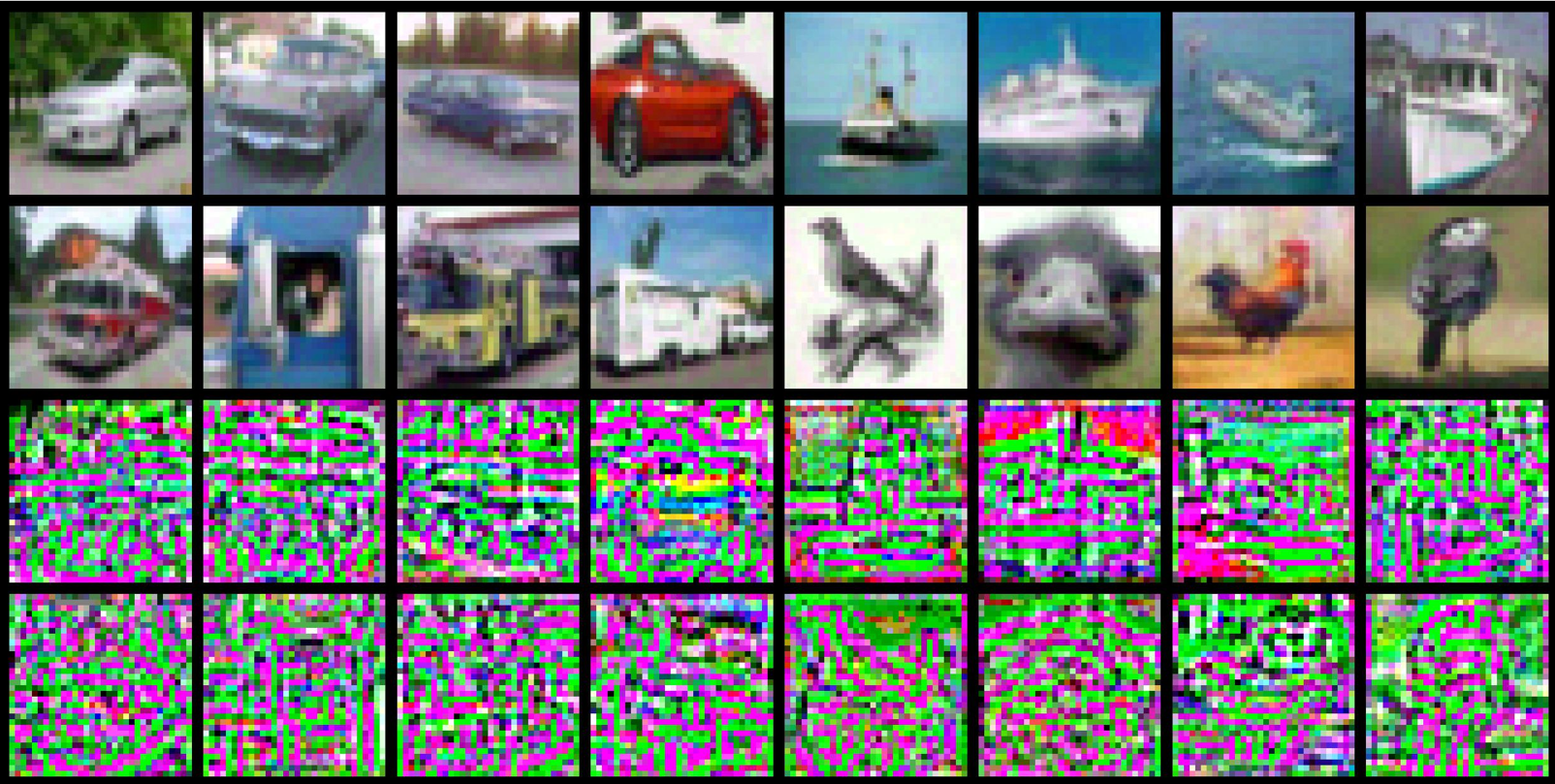}
\caption{RobustEM poisoned CIFAR-10 and its corresponding perturbations.}
\end{figure}
\end{minipage}

\begin{minipage}[ht]{0.45\textwidth}
\begin{figure}[H]
\centering
\hspace{3mm}
\includegraphics[width=6.0truecm]{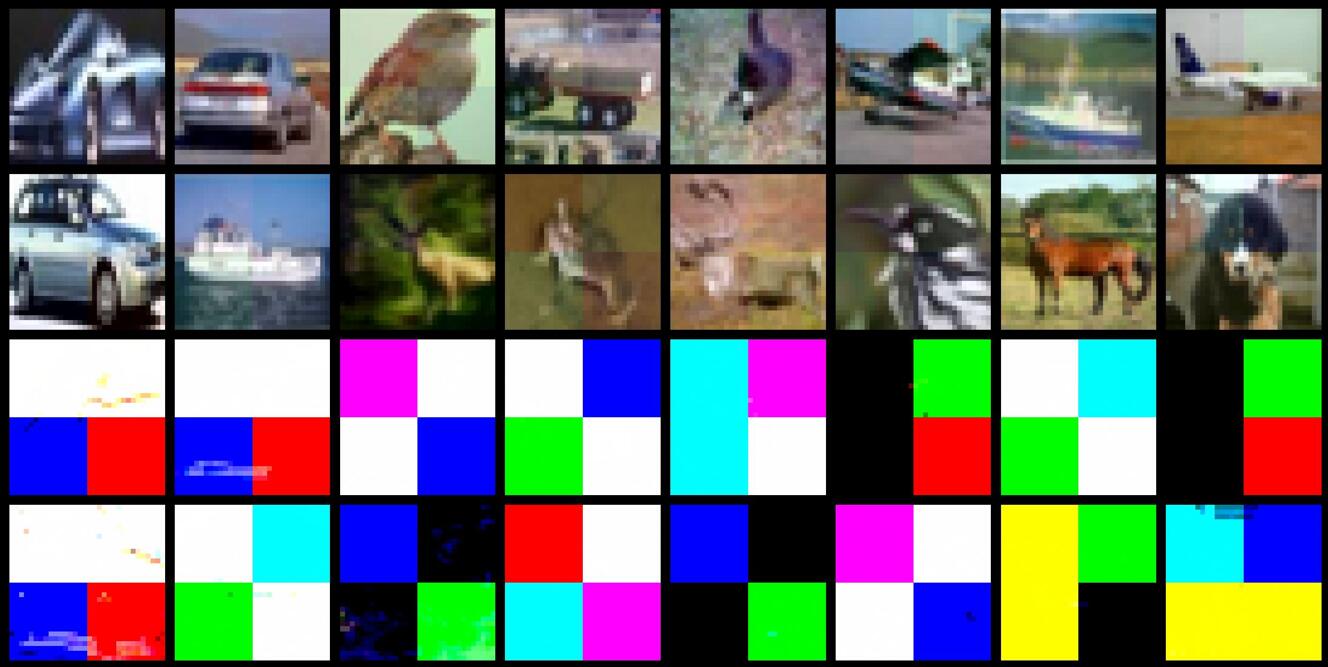}
\caption{Region-4 poisoned CIFAR-10 and its corresponding perturbations.}
\end{figure}
\end{minipage}
\quad
\begin{minipage}[ht]{0.45\textwidth}
\begin{figure}[H]
\centering
\includegraphics[width=6.0truecm]{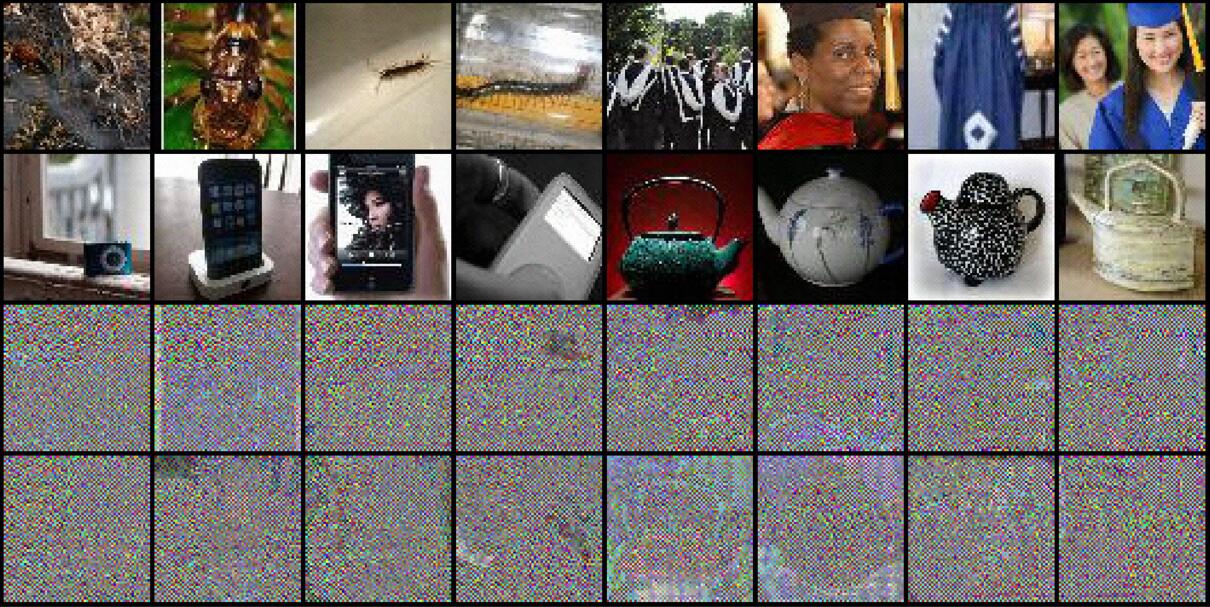}
\caption{Err-min(S) poisoned TinyImageNet and its corresponding perturbations.}
\end{figure}
\end{minipage}

\section{Abnormal styles of unlearnable examples}
\label{sec-7.1}
To better understand how and why unlearnable examples work, we track their training procedure, finding some anomalous performance on peak accuracy and learning curves, which reveals abnormal styles of unlearnable examples, that they learn injected noises rather than original images.

\subsection{Peak accuracy of unlearnable examples}
\label{obs-power}

We evaluate the unlearnable examples given in Section \ref{poison} on several frequently-used networks 
VGG16 
\cite{simonyan2014very}
, ResNet18 
\cite{he2016deep}
, WRN34-10 
\cite{zagoruyko2016wide}, 
DenseNet121 
\cite{huang2017densely} 
and ViT 
\cite{dosovitskiy2020image} 
under different initial learning rates, 0.1, 0.01, 0.001 and 0.0001.
The experimental results are given in Tables \ref{tab-e11} and \ref{tab-e12}.

Based on Tables \ref{tab-e11} and \ref{tab-e12}, we find that for commonly-used initial learning rate 0.1, most poisoned data reach the random guess level, i.e., about $10\%$ for CIFAR-10 dataset.
However, when using smaller learning rates and early-stopping, {\em all of these methods reach much higher peak accuracy}, some methods like Random(C) and AR even reach over $60\%$ peak accuracy, almost making these unlearnable examples back to be learnable.

\begin{table}[!htbp]
\caption{Victim models test accuracy (\%) for poisoned CIFAR-10 under some unlearnable examples.
{\bf Highest} means the highest recorded test accuracy under four types of learning schedules, with initial learning rate be 0.1, 0.01, 0.001 and 0.0001.
{\bf Final} means the test accuracy using initial learning rate be 0.1 and the final result after 100 epochs.
Detailed learning schedules are referred in Appendix \ref{sec-learn-detail}.
}
\vskip5pt
\label{tab-e11}
\setlength{\tabcolsep}{4.5pt}
\centering
\begin{tabular}{lccccccccccc}
\toprule
  \multicolumn{1}{c}{Victim Model} &\multicolumn{2}{c}{VGG16} & \multicolumn{2}{c}{ResNet18}& \multicolumn{2}{c}{WRN34-10} & \multicolumn{2}{c}{DenseNet121} & \multicolumn{2}{c}{ViT}\\
\midrule
  Poison & Final & Highest  & Final & Highest  & Final & Highest  & Final & Highest & Final & Highest\\
  Random(C) &  10.32&69.47&11.02&79.30 & 9.95&81.84&80.63&80.78&55.71&58.10\\
  Region-4 & 10.69&35.48&13.90&32.06 & 19.43&34.09&23.59&34.15&21.79&30.13\\
  Region-16 &  10.84&31.29&19.86&35.81& 10.12&39.64&20.90&44.67&21.43&31.00\\
  Err-min(S) &9.43&45.20 &10.09&40.80& 10.10& 35.84& 10.05&47.81&10.35&56.64\\
  Err-min(C)&10.00&51.24&10.01&49.03& 10.31&38.04&11.61&48.37&26.22&63.42\\
  Err-max&8.21&62.40&7.19&62.03&8.39&50.70&39.77 &68.28&64.15&64.21 \\
  NTGA&10.82&46.75&11.23&45.01& 13.58&44.70&36.23&48.24&30.66&54.11\\
  AR &14.29 & 67.34 & 17.18 & 65.70&  16.51 & 66.23 & 82.51 & 82.65 & 45.97 & 67.30\\
  RobustEM&24.25&45.39&25.30&41.53& 21.56&38.08&27.94&48.55&19.17&50.08\\
\bottomrule
\end{tabular}
\end{table}

\begin{table}[!ht]
\caption{The final and highest test accuracy (\%) for victim models when training in poisoned CIFAR-100 and TinyImageNet under some unlearnable examples.}
\vskip5pt
\label{tab-e12}
\centering
\begin{tabular}{clcccccc}
 \toprule
\multirow{2}{*}{Dataset}& Victim model & \multicolumn{2}{c}{ResNet18} & \multicolumn{2}{c}{WRN34-10} & \multicolumn{2}{c}{ViT}\\
 \cline{2-8}
 
   &Poison & Final & Highest  & Final & Highest  & Final & Highest  \\
\midrule
  \multirow{6}{*}{CIFAR-100}
  &Random(C)& 53.59&55.47&2.42&39.19&36.17&41.05\\
  &Region-4 & 3.95&13.17&2.83&12.60&8.67&12.92\\
  &Region-16 & 1.06&20.37&1.09&16.44&6.74&15.67\\
  &Err-min(S) & 1.00&34.66&1.02&36.18&16.85&37.07\\
  &Err-min(C) &0.75&32.98&1.14&33.05&18.23&37.47\\
  \hline

 \multirow{6}{*}{TinyImageNet}
  &Random(C) & 31.74&35.23&37.99&38.18&14.39&16.99\\
  &Region-4 & 7.54&10.80&
  5.94&15.10&9.48&12.69\\
  &Region-16 & 7.18&12.72&3.76&25.76&6.71&13.42\\
  &Err-min(S) & 1.72&28.18&0.92&24.45&24.13&24.21\\
  &Err-min(C) &1.37&23.01&1.07&18.21&9.14&13.15\\
\bottomrule
\end{tabular}
\end{table}

\subsection{Unusual learning curves of poisoned data}
\label{sec-obs}
\begin{figure}[!htt]
\centering
\hspace{2mm}
\subfigure[lr $=1\times10^{-2}$]{\includegraphics
[scale=0.35]{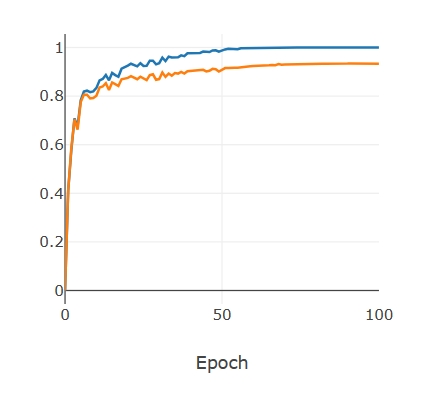}}
\subfigure[lr $=1\times10^{-3}$]{\includegraphics
[scale=0.35]{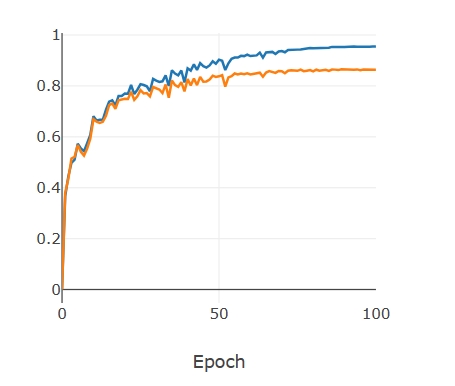}}
\subfigure[lr $=1\times10^{-4}$]{\includegraphics
[scale=0.35]{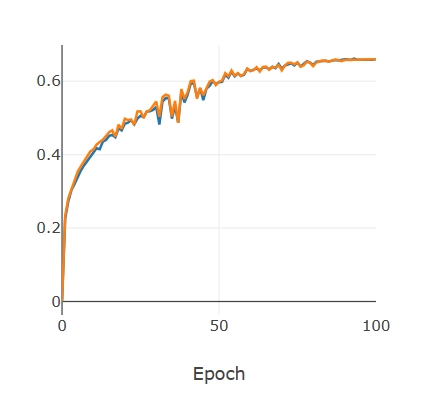}}
\\
\caption{Learning curves of clean CIFAR-10.
The victim model is ResNet18. {\bf lr} means the initial learning rate. Blue curves and orange curves means the training and test accuracy respectively.
}
\label{fg3}
\end{figure}

In Figures \ref{fg3}, \ref{fg-errmins}, \ref{fg-errminc} and \ref{fg-region}, the learning curves on the normal training set and some poisoned training set for different initial learning rates are shown.
From Figures \ref{fg-errmins}, \ref{fg-errminc} and \ref{fg-region}, we can see that the training accuracy of poisoned dataset quickly goes to $100\%$ for learning rates 0.01 and 0.001, but the test accuracy does not rise, or rises slightly and then quickly falls back to a very low level.
On the contrary, for the clean training set, the test accuracy grows along with the training accuracy.

Furthermore, when the initial learning rate is much smaller, like 0.0001, the peak test accuracy when training on poisoned dataset is much higher than that of  0.01. 
However, training on clean dataset with initial learning rate 0.0001 does not converge as shown in Figure  \ref{fg3}(c), 
 which means that this learning rate is not enough for the network to learn the original features.
The training curves of poisoned dataset rise very fast, even for a very small learning rate that clean dataset can not be learned well, they still achieve 100\% training accuracy. 
In conclusion, the victim models indeed learn something useful under poisoned dataset, but quickly they are forgotten.

\begin{figure}[!ht]
\centering
\hspace{2mm}
\subfigure[lr $=1\times10^{-2}$]{
\includegraphics[scale=0.35]{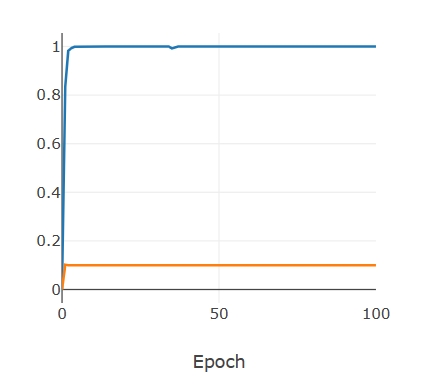}}
\subfigure[lr $=1\times10^{-3}$]{
\includegraphics[scale=0.35]{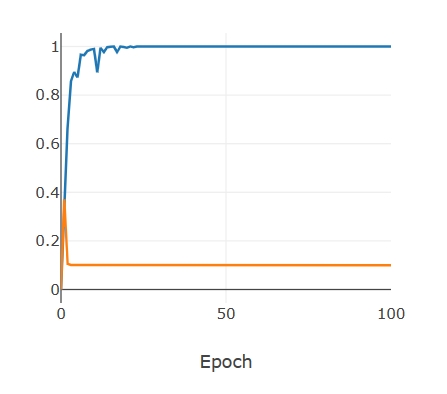}}
\subfigure[lr $=1\times10^{-4}$]{
\includegraphics[scale=0.35]{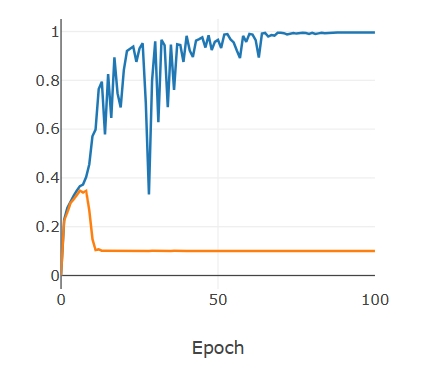}}\\

\caption{Learning curves of Err-min(S) poisoned CIFAR-10.
 Notations are same to Figure \ref{fg3}.
}
\label{fg-errmins}
\end{figure}

\begin{figure}[H]
\centering
\hspace{2mm}
\subfigure[lr $=1\times10^{-2}$]{
\includegraphics[scale=0.35]{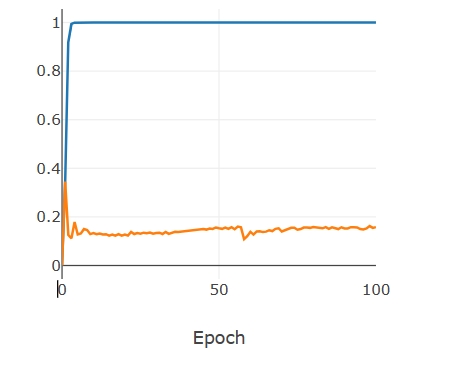}}
\subfigure[lr $=1\times10^{-3}$]{
\includegraphics[scale=0.35]{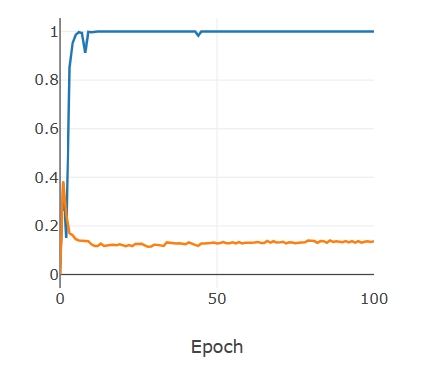}}
\subfigure[lr $=1\times10^{-4}$]{
\includegraphics[scale=0.35]{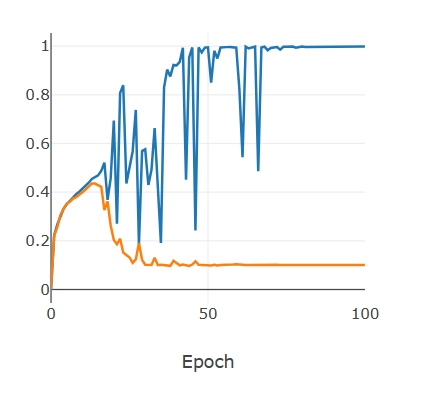}}

\caption{Learning curves of Err-min(C) poisoned CIFAR-10. 
Notations are same to Figure \ref{fg3}.}
\label{fg-errminc}
\end{figure}

\begin{figure}[H]
\centering
\hspace{2mm}
\subfigure[lr $=1\times10^{-2}$]{
\includegraphics[scale=0.35]{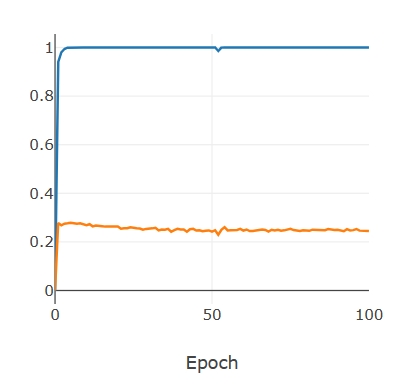}}
\subfigure[lr $=1\times10^{-3}$]{
\includegraphics[scale=0.35]{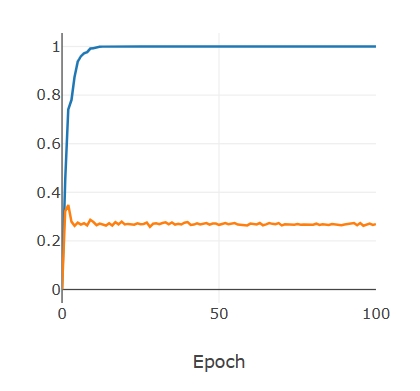}}
\subfigure[lr $=1\times10^{-4}$]{
\includegraphics[scale=0.35]{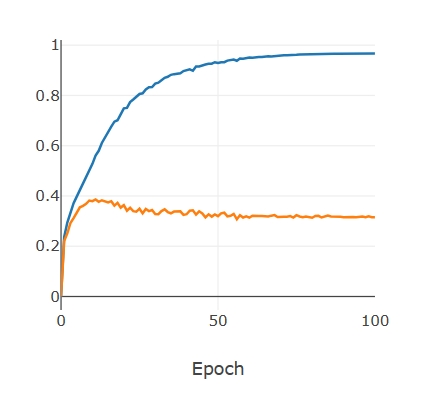}}

\caption{Learning curves of Region-16 poisoned CIFAR-10. 
Notations are same to Figure \ref{fg3}.}
\label{fg-region}
\end{figure}

\subsection{Dominance of noise features on unlearnable examples}
\label{sec-7.2}
In this section, extensive experiments demonstrate that training on unlearnable examples will result in the network to learn the features of poisons rather than the original images.

Let $D_{\ts}=\{(\boldsymbol{x}_i,y_i)\}$ be the test set (or validation set), $D_{\ts}^{\poi}=\{(\boldsymbol{x}_i+\boldsymbol{\epsilon}(\boldsymbol{x}_i),y_i)\}$ be the test set poisoned by unlearnable attacks.
We consider 8 more type of test sets (or validation sets):
\begin{itemize}
\item[(1)] $D_{\ts}^{\poi}(1)=\{(\boldsymbol{\epsilon}(\boldsymbol{x}_i),y_i)\}$;
\item[(2)] $D_{\ts}^{\poi}(2)=\{(\boldsymbol{\epsilon}(\boldsymbol{x}_i)+0.5\boldsymbol{e},y_i)\}$;
\item[(3)] $D_{\ts}^{\poi}(3)=\{(\boldsymbol{\epsilon}(\boldsymbol{x}_i)-0.5\boldsymbol{e},y_i)\}$;
\item[(4)] $D_{\ts}^{\poi}(4)=\{(\boldsymbol{\epsilon}(\boldsymbol{x}_i)+\boldsymbol{x}_j,y_i)\}$, $\boldsymbol{x}_j$ is a randomly selected sample with label different from $\boldsymbol{x}_i$;
\item[(5)] $D_{\ts}^{\poi}(5)=\{(\boldsymbol{\epsilon}(\boldsymbol{x}_i)+\boldsymbol{x}_j,y_j)\}$, $\boldsymbol{x}_j$ is a randomly selected sample with label different from $\boldsymbol{x}_i$;
\item[(6)] $D_{\ts}^{\poi}(6)=\{(\boldsymbol{\epsilon}(\boldsymbol{x}_i)+32/255 \delta(-1,1)\boldsymbol{e},y_i)\}$, $\delta(-1,1)=\pm 1$ with probability 0.5 respectively.
\item[(7)] $D_{\ts}^{\poi}(7)=\{(\boldsymbol{\epsilon}(\boldsymbol{x}_i)+U(-32/255,32/255)\boldsymbol{e},y_i)\}$, $U$ is uniform distribution.
\end{itemize}

We use ResNet18 as the victim model to conduct the noise learnability test.
Results are shown on Table \ref{tab-noise-learn}.
$D_{\ts}^{\poi}(2)$ and $D_{\ts}^{\poi}(3)$ represent the noise learnability under a large bias-shifting noise.
$D_{\ts}^{\poi}(4)$ and $D_{\ts}^{\poi}(5)$ represent adding a poison on a clean data or adding a clean data on a poison, with different labels.
$D_{\ts}^{\poi}(6)$ and $D_{\ts}^{\poi}(7)$ represent adding a large magnitude of the random perturbations.

\begin{table}[H]
\caption{Noise learnability (\%) of some poison methods for CIFAR-10.}
\label{tab-noise-learn}

\setlength{\tabcolsep}{4pt}
\centering
\begin{tabular}{lcccccccccc}
  \toprule
  Poison & $D_{\ts}$ & $D_{\ts}^{\poi}$  &$D_{\ts}^{\poi}(1)$ & $D_{\ts}^{\poi}(2)$  &$D_{\ts}^{\poi}(3)$    & $D_{\ts}^{\poi}(4)$ & $D_{\ts}^{\poi}(5)$& $D_{\ts}^{\poi}(6)$& $D_{\ts}^{\poi}(7)$\\
  Random(C)  & 10.46 & 100.0 & 100.0 & 100.0 & 100.0 & 100.0 & 0.00 & 17.96 & 32.17\\
  Region-4 & 17.21 & 99.97 & 100.0 & 100.0 & 100.0 & 99.87 & 0.01 & 75.58 & 97.74\\
  Region-16 & 15.27 & 100.0 & 99.99 & 100.0 & 100.0 & 99.94 & 0.02 & 78.78 & 98.25\\
  \text{Err-min}(S) & 10.00 & 99.99 & 99.98 & 100.0 & 100.0 & 99.91 & 0.00 & 84.84 & 89.38\\
  \text{Err-min}(C) & 9.60 & 100.0 & 100.0 & 100.0 & 100.0 & 100.0 & 0.00 & 99.20 & 99.98\\
  NTGA & 10.10 & 99.98 & 99.98 & 100.0 & 100.0 & 99.80 & 0.06 & 11.10 & 15.24\\
  AR & 13.63 & 99.98 & 100.0 & 100.0 & 100.0 & 99.53 & 0.18 & 10.14 & 12.65\\
\bottomrule
\end{tabular}
\end{table}

The results demonstrate that all of the unlearnble examples have powerful competence of learning the deliberately injected poisons, even for large bias-shifting noise like $D_{\ts}^{\poi}(2)$ and $D_{\ts}^{\poi}(3)$, and adding images with different labels like $D_{\ts}^{\poi}(4)$.

For $D_{te}^{\poi}(5)$,  the sample $\boldsymbol{\epsilon}(\boldsymbol{x}_i)+\boldsymbol{x}_j$ should correlate to label $y_j$. However, the poison $\boldsymbol{\epsilon}(\boldsymbol{x}_i)$ is so strong to break the relationship
between $\boldsymbol{x}_j$ and $y_j$.
It is worth noting that comparing to $D_{te}^{\poi}(4)$, the network always predicts labels $y_i$ to $\boldsymbol{\epsilon}(\boldsymbol{x}_i)+\boldsymbol{x}_j$ but not $y_j$, which means that it focuses on the noise feature $\boldsymbol{\epsilon}(\boldsymbol{x}_i)$ rather than the original one $\boldsymbol{x}_j$.

Datasets $D_{\ts}^{\poi}(6)$ and $D_{\ts}^{\poi}(7)$ are provided to show that some injected poisons are very strong and robust like Region-$k$ and Err-min(S,C) poisons.
For $D_{\ts}^{\poi}(6)$ and $D_{\ts}^{\poi}(7)$,
the perturbation budget ($32/255$) is four times of the poison injection budget ($8/255$), but these poisons are still learnable as shown by the results in Table \ref{tab-noise-learn}. 
This further reveals that the injected noises have strong correlation with labels, which are very hard to be destroyed.

\section{Time of experiments}
\label{poi-time}

\begin{table}[H]
\caption{Poison generation time for CIFAR-10, with a single NVIDIA 3090 GPU.}
\label{tab-e4}
\centering
\begin{tabular}{llccccc}
\toprule
Poison Attack & Time \\
\midrule
  Random(C) & $\approx 3$ s\\
  Region-4 &  $\approx 30$ s\ \\
  Region-16 &   $\approx 30$ s\\
  Error-min(S) & $\approx 90$ min\\
  Error-min(C) & $\approx 15$ min\\
  NTGA & $\approx 10$ h\\
  RobustEM & $\approx 50$ h\\
\bottomrule
\end{tabular}
\end{table}

From Table \ref{tab-e4}, we can see that
the poison generation time of class-wise poisons
are much smaller than that of sample-wise poisons,
which is rational since the numbers of classes are
much smaller than those of samples.

\begin{table}[H]
\caption{Time of unlearnable examples defense for CIFAR-10 and CIFAR-100, with a single NVIDIA 3090 GPU. \textbf{AN+SDA} means using stronger data augmentations with adversarial noises generated by two-layer NN, \textbf{AT} means vallina adversarial training.}
\label{tab-noise-time}
\centering
\begin{tabular}{lccc}
\toprule
Defense method & AN+SDA & AT \\
\midrule
  CIFAR-10 Time & $\approx 80$ min & $\approx 270$ min\\
  CIFAR-100 Time & $\approx 80$ min & $\approx 480$ min\\
\bottomrule
\end{tabular}
\end{table}

\section{The similarity of sample-wise poisons for the same label}
\label{sec-7.3}

As shown in Table \ref{tab-e11}, even for the Random(C) unlearnable example, the victim models can be surprisingly fooled under usual learning schedules. 
The effective unlearnable examples Region-$k$ and Err-min(C) poisons are also class-wise-based.
These facts reveal the great power for class-wise noise to induce unlearnability, because victim models are deceived to learn the different noises as features for classification rather than the image features themselves.
Therefore, it is reasonable to guess that sample-wise unlearnable examples generate similar noises for samples within the same class, which behave like class-wise unlearnable examples. 
In this section, we will analyze the similarity of sample-wise poisons for samples within the same class using two measurements.

{\bf Cosine similarity of poisons between intra-classes and inter-classes.}
The cosine similarity between two poisons  $\boldsymbol{\epsilon}(\boldsymbol{x}_i)$ and $\boldsymbol{\epsilon}(\boldsymbol{x}_j)$ is $\textup{cossim}(\boldsymbol{\epsilon}(\boldsymbol{x}_i),\boldsymbol{\epsilon}(\boldsymbol{x}_j))=\frac{\left<\boldsymbol{\epsilon}(\boldsymbol{x}_i),\boldsymbol{\epsilon}(\boldsymbol{x}_j)\right>}{\|\boldsymbol{\epsilon}(\boldsymbol{x}_i)\|_2\|\boldsymbol{\epsilon}(\boldsymbol{x}_j)\|_2}$, where $\left<\cdot,\cdot\right>$ is the inner product.
We use the Monte Carlo method to randomly sample 10000 images from two datasets and take the average of them as the approximation of cosine similarity between the two sets. 
Results are shown in Table \ref{cossim}.

 \begin{table}[H]
\caption{Cosine similarity of noise under sample-wise unlearnable examples.
Intra-class means for samples within the same class.
Inter-class means for samples from different classes.
}
\label{cossim}
\centering
\begin{tabular}{lcc}
 \toprule
  Poisons  & Intra-classes & Inter-classes \\
\midrule
  Random noise   & 0.0000 &-0.0003\\
  Err-min(S)   & 0.1232 &0.0024\\
  Err-max & 0.0253 & 0.0036 \\
  NTGA   & 0.7757 &0.2317 \\
  AR   & 0.0012 & -0.0001 \\
\bottomrule
\end{tabular}
\end{table}

{\bf Commutative KL divergence of poisons between intra-class and inter-class.}
We assume that training samples obey the multivariate Gaussian distribution for simplicity, because the KL divergence between two Gaussian distributions has close form solutions.
We estimate the mean and covariance of multivariate Gaussian distributions with existing poisons.
For intra-classes, we randomly spilt poisons in the same class into two parts with the same amounts, then estimate their means and covariances respectively to obtain two multivariate Gaussian distributions and their KL divergences.
Since KL divergence is non-commutative, we use commutative KL divergence defined as $\textup{CKL}(A,B)=\frac{1}{2}(\textup{KL}(A,B)+\textup{KL}(B,A))$. Results are shown in Table \ref{ckl}.

\begin{table}[htbp]
\caption{Commutative KL divergence of poisons under sample-wise unlearnable examples.
Intra-class means for samples within the same class.
Inter-class means for samples from different classes.
}
\label{ckl}
\centering
\begin{tabular}{lcc}
\toprule
Poisoned Data  & Intra-classes & Inter-classes \\
\midrule
  Random noise & 0.0001& 0.0002\\
  Err-min(S)  & 0.0000 & 0.9961\\
  Err-max & 0.0004 & 0.0314 \\
  NTGA  & 0.0000 & 1.7677\\
  AR  & 0.0002 & 0.0353 \\
\bottomrule
\end{tabular}
\end{table}

As shown by Tables \ref{cossim} and \ref{ckl}, although sample-wise poisons have different injected poisons for each sample, the similarity for samples from within the same class are significantly larger than those from different classes, which reveals sample-wise poisons implicitly contain properties of class-wise injected poisons.

\section{More experiments on detection of unlearnable examples}
\label{sec-7.4}

In this section, we give more experimental results for Algorithms \ref{alg-weak} and \ref{alg-rule} to show their power for detection with more unlearnable examples, including some indiscriminate attacks violating the restriction of unlearnable examples on poison power, such as CUDA \cite{sadasivan2023cuda} and OPS \cite{wu2022one}.

\subsection{Algorithm \ref{alg-weak}: Simple Networks Detection}
We use a linear model or a two-layer NN to evaluate the training accuracy on clean dataset and various unlearnably poisoned dataset.
For two-layer NN we set the width of hidden layer be same as the dimension of the input.

\begin{table}[!ht]
\caption{The detection accuracy (\%) for poisoned CIFAR-10.
For Err-min(S, C), we use VGG16 and ResNet18 respectively as the source model to generate poisons.
}
\label{detect1-cifar10}
\centering

\begin{tabular}{lccccccccc}
  \toprule
  \multirow{2}{*}{Poison Methods} & \multicolumn{2}{c} {Simple Networks Detection} & \multicolumn{2}{c} {Bias-shifting Noise Test}\\

  & \multicolumn{1}{c} {Linear model} & \multicolumn{1}{c} {Two-layer NN} & $\epsilon_b=-0.5$ & $\epsilon_b=0.5$\\
\midrule
  Clean CIFAR-10  &\textit{46.53} &\textit{57.33} &\textit{49.08} &\textit{42.12}\\
  Random(C)   &  100.0 &99.13 & 100.0 & 99.84\\
  Region-4  & 97.15& 99.82 & 98.74 & 99.84\\
  Region-16 & 99.87& 99.98 & 99.98 & 99.64\\
  Region-64 & 100.0& 100.0 & 100.0 & 99.64\\
  Err-min(VGG16,S)  &  89.43 &98.83 & 98.56 & 96.26\\
  Err-min(ResNet18,S)  & 99.99&99.77& 97.85 & 100.0\\
  Err-min(VGG16,C) &  100.0&100.0 & 99.33 & 99.90\\
  Err-min(ResNet18,C) &100.0&99.66&100.0&100.0\\
  Err-max & 80.48 & 98.52 & 99.96 &99.20\\
  NTGA  &  97.39&99.13&90.82&95.26\\
  AR & \textit{46.96}& 96.75&90.96&99.98\\
  RobustEM & 78.49 & 99.41 &99.10 &99.40 \\
  Hypocritical & 77.48 & 99.66 &98.10 &94.06\\
  Adv Inducing & \textit{53.37} & 98.94 &86.25 &75.48 \\
  CP & \textit{47.49} & \textit{56.87} & 81.31 & 78.28\\
  TUE & 100.0 & 78.85 & 99.92 & 99.37 \\
  CUDA & 97.89 & 80.40 & \textit{13.62} & \textit{25.08} \\
  OPS & 99.94 & 99.78 & 98.66 & 85.13 \\
\bottomrule
\end{tabular}
\end{table}

\subsection{Algorithm \ref{alg-rule}: Bias-shifting Noise Test}
\label{app-bias-shift}

As mentioned in Section \ref{sec-exp11}, most of the poisoned dataset could be detected by the linear model and bias-shifting noise test with $\boldsymbol{\epsilon}_b=0.5\boldsymbol{e}$, and all of them could be detected by two-layer NN and bias-shifting noise test with $\boldsymbol{\epsilon}_b=-0.5\boldsymbol{e}$.
The experimental results in this section further supported this statement.
The datasets recognized as a clean dataset by detection algorithms if they are marked in italic. 
The detection bound $B$ is 0.7.

\paragraph{Choices of bias-shifting noises.}
We evaluate more choices of bias-shifting noises within $[-0.9, 0.9]$ (-1 and 1 are meaningless as they would result in a complete degradation of the image to 0 and 1), The results are provided in the following table.
 \begin{table*}[!ht]
 \caption{Detection accuracy of different magnitude of bias-shifting noises under Err-min(S) unlearnable examples compared with the clean CIFAR-10. 
}
\label{bias-noise-choice}
 \small
\centering
\setlength{\tabcolsep}{2.2pt}
\begin{tabular}{lccccccccccccccccccc}
 \toprule
Noise & -0.9 & -0.8& -0.7 & -0.6 &-0.5 & -0.4 & -0.3 & -0.2 & -0.1 & 0.0 & 0.1 & 0.2 & 0.3 & 0.4 & 0.5 & 0.6 & 0.7 & 0.8 & 0.9\\
  \midrule
  Clean & \textit{16.97} & \textit{25.58} & \textit{26.49} & \textit{41.53} & \textit{49.08} & \textit{60.73} & \textit{67.86} & 80.84 & 77.57 & 78.72 & 78.22 & 75.69 & 71.70 & \textit{60.45} & \textit{42.12} & \textit{34.43} & \textit{25.97} & \textit{25.11} & \textit{15.40} \\
  Poisoned & 95.65 & 92.65 & 96.83 & 97.95 & 97.85 & 100.0 & 100.0 & 100.0 & 100.0 & 100.0 & 100.0 & 100.0  & 100.0 & 100.0 & 100.0  & 99.98 & 97.76 & 96.78 & 88.30  \\

  \bottomrule
\end{tabular}
\end{table*}

\section{More experiments on stronger data augmentations defense}
\label{strong-aug-app}

\paragraph{Defense power under different strengths of data augmentations.}
We also evaluate defense power of stronger data augmentations of different strengths without using predetermined adversarial-trained noises on Table \ref{tab-aug-def-nodenoise}. 
Results in Table \ref{tab-aug-def-nodenoise} also show that only pure data augmentation method can also achieve comparable defense power to adversarial training.

Follow the setting of \cite{luo2023rethinking}, we quantify the power of data augmentation with strength $s\in[0, 1]$.
When $s=0$, the augmentation degrades to the data augmentations with random resized crop and random horizontal flip.
When $s=1$, the augmentation equals to the common contrastive learning setting with random resized crop, random horizontal flip, color jitter and random gray-scale.

\begin{table*}[!ht]
\caption{Test accuracy (\%) under different data augmentation strengths for poisoned CIFAR-10 without using adversarial-trained noises.}
\label{tab-aug-def-nodenoise}
\centering
\begin{tabular}{cccccc}
 \toprule
Strength & 1.0 & 0.8 & 0.6& 0.4 & 0.2\\
  \midrule
  Region-16 & \textbf{77.20}  & 74.30  & 69.68 & 59.53 & 46.95\\
  Err-min(S) &  \textbf{75.62} & 73.22  & 72.63 & 66.77 & 44.69\\
  RobustEM &  79.21 & \textbf{80.55}  & 78.44 & 72.19 & 63.72\\
  \bottomrule
\end{tabular}
\end{table*}

\paragraph{Defense power on CIFAR-100.}
We evaluate our defense Algorithm \ref{alg-strong-aug} on three typical kinds of unlearnable examples on CIFAR-100.
Compared to adversarial training, results show that our defense method outperforms all of the three kinds of unlearnable examples, and even only use stronger data augmentations, the test accuracy are comparable to adversarial training. 
The experimental results further support our claim in Section \ref{strong-aug}, that destroying the detectability rather than achieving robustness can help people defend against unlearnable examples.

\begin{table*}[!ht]
\caption{Test accuracy (\%) of stronger data augmentations with adversarial noises for poisoned CIFAR-100.}
\label{tab-aug-def-cifar100}
\centering
\begin{tabular}{cccccc}
 \toprule
Method/Poison & Region-16 & Err-min(S) & RobustEM\\
  \midrule
  AT-based methods & 53.08 & 52.23 & 49.85 \\
  Adversarial noises (AN) &  24.83 & 5.26  & 7.12 \\
  Stronger data augmentation (SDA) &  44.60 & 49.82  & \textbf{50.61} \\
  AN+SDA &  \textbf{72.86} & \textbf{71.56}  & 49.40 \\
  \bottomrule
\end{tabular}
\end{table*}

\paragraph{Suboptimality on Err-max poison.}
In order to gain insights into the reasons behind the suboptimal performance of our defense method when applied to Err-max, we have undertaken several simple ablation studies on them.
In addition to examining the effects of adversarial noise (AN), we have also explored the utilization of anti-adversarial noise (Anti-AN) applied to a two-layer neural network. 
This involves identifying the loss function's minimum point through a PGD attack. Furthermore, we have taken into account general data augmentation (DA), encompassing random cropping and random horizontal flipping. The results are shown in Table \ref{diff-aug-errmax}.

\begin{table*}[!ht]
\caption{Defense performance of different data augmentations and adversarial noises.}
\label{diff-aug-errmax}
\centering
\begin{tabular}{lccccccc}
 \toprule
Defense Method & DA & AN+DA & AntiAN+DA & SDA & AN+SDA & AntiAN+SDA\\
  \midrule
  Performance & 14.23 & 28.74 & 63.99 & 56.96 & 61.19 & 74.36 \\

  \bottomrule
\end{tabular}
\end{table*}

The obtained results underscore a noteworthy observation – the efficacy of adversarial noise (AN) as a defense mechanism for Err-max poisons is rather limited.
AN displays only a marginal improvement over employing no denoising procedure at all and falls short when compared to the application of anti-adversarial noise (AntiAN) as a denoising technique. 
This outcome could potentially be attributed to the nature of Err-max poisons, which employ adversarial noises in their unlearnable attacks. 
Consequently, attempting similar adversarial noise application on simple networks might yield suboptimal results in terms of denoising the perturbation. 
While AntiAN endeavors to mitigate Err-max perturbations, it remains incapable of eliminating their detectability as shown in Table \ref{detection-errmax}, resulting in a performance lag of approximately 10\% compared to the state-of-the-art defense method.

\begin{table*}[!ht]
\caption{Detection performance on Err-max poisons with different adversarial noises.}
\label{detection-errmax}
\centering
\begin{tabular}{lcccc}
 \toprule
Detection Method & 2-NN & Bias-0.5\\
  \midrule
  Err-max & 80.48 & 98.52 & \\
  Err-max+AN & \textit{32.53} & \textit{37.81} & \\
  Err-max+AntiAN & 90.40 & 97.33 & \\

  \bottomrule
\end{tabular}
\end{table*}
Therefore, as shown in Table \ref{tab-aug-def1}, AT-based methods are still the most effective defense methods for Err-max poisons, we may leave the more powerful defense on Err-max as the future works.

\section{More discussions and experiments on defense by adversarial training}
\label{sec-7.5}
\subsection{More theorems and discussions on criteria between poison budget and defense budget}
\label{def-th-app}
As shown in Section \ref{sec-7.2}, unlearnable examples tend to learn the injected poisons rather than original features of image. 
As noises are imperceptible and bounded by a small budget, it is reasonable for people to believe that adversarial training will bring effective defense because they force the model to perform well on all of the perturbed training data within a small perturbation budget. 
Theoretically, the following theorem indicates that adversarial training is minimizing an upper bound of population loss of the poisoned dataset under  Wasserstein distance.

\begin{theorem}[\cite{tao2021better}]
\label{th-o20}
Let $\D$ be a data distribution, $\D^{\poi}$ be the poisoned dataset of $\D$ under $l_{\infty}$ norm Wasserstein distance, that is, $\D^{\poi}\in\B_{W_{\infty}}(\D,\epsilon)$.
Then for a classifier $\F$, it holds that $\mathcal{R}_{\textup{natural}}(\F,\D)\leq\mathcal{R}_{\textup{robust}}(\F,\D^{\poi}),$ where $\mathcal{R}_{\textup{natural}}$ and $\mathcal{R}_{\textup{robust}}$ denote the natural risk and robust risk (with robust radius $\epsilon$), respectively.
\end{theorem}

\paragraph{Similarity between victim model and linear model.}
To further show the high linearity of the victim model under unlearnable examples, we conducted a simple experiment presented in Remark \ref{rema-linear},
which indicates that if $\B_p(S,\epsilon)$ is linearly separable, adversarial trained networks will perform similarly to a linear model.
Therefore, since normal datasets are not linear separable according to Remark \ref{rem-58}, the victim model will achieve poor test accuracy even under adversarial training regimes.

\begin{remark}
\label{rema-linear}
We evaluate the linearity of the network trained by Err-min(C) poisoned and clean CIFAR-10, using mean square loss to evaluate the logit of mixup data $\F(\lambda \boldsymbol{x}_i +(1-\lambda) \boldsymbol{x}_j)$ and the mixup logit $\lambda\F(\boldsymbol{x}_i)+(1-\lambda)\F(\boldsymbol{x}_j)$.
Results of the victim model and standard model are 0.0002 and 0.0189, respectively, which shows much higher linearity of the model trained by the poisoned dataset than the clean one.
\end{remark}

\subsection{More experimental observations on poison budget and defense budget}
\label{poison-and-def}
In Table \ref{tab-adv21}, we give the test accuracy
under adversarial training regimes with different budgets trained on poisoned dataset to further support Section \ref{sec-experiment2}.

\begin{table}[ht]
\caption{Test accuracy (\%) for using adversarial training budgets $i/255,i=0,\ldots,16$ to defend unlearnable examples with  poison budget $\epsilon=8/255$. We do not use any data augmentation here for better verification of our theorems. }
\label{tab-adv21}
\centering
\setlength{\tabcolsep}{4.5pt}

\begin{tabular}{cclccccccccc}
 \toprule
&Attack & Poison & 0 & 1/255 & 2/255 &3/255& 4/255 & 6/255 & 8/255 & 12/255 & 16/255\\
\midrule
  \multirow{6}{*}{\begin{turn}{90}CIFAR-10\end{turn}}& \multirow{3}{*}{PGD-10}
  & Region-16&19.86 & 24.32 & 29.57 &50.09& 72.13&\textbf{77.03}&72.65 &67.61&62.78\\
  &&Err-min(S)&10.09 &10.01&10.13& 18.64&69.92&\textbf{76.53} &72.13&67.23&61.70\\
  &&RobustEM&25.30 &24.92&28.50&33.74&46.01&\textbf{76.69} &72.19&63.16&53.34\\
  &\multirow{3}{*}{FGSM} & Region-16&19.86 & 18.97 & 23.09 &31.77& 43.80&58.67&\textbf{75.80} &72.94&64.13\\
  &&Err-min(S)&10.09 &10.00&10.00&10.01&10.57&23.70 &70.78&\textbf{71.04}&63.29\\
  &&RobustEM&25.30 &27.52&26.37&26.95&32.14&47.16 &65.07&\textbf{72.83}&61.18\\
\midrule
  \multirow{6}{*}{\begin{turn}{90}CIFAR-100\end{turn}} & \multirow{3}{*}{PGD-10}
  &Region-16& 1.06 & 2.16 & 9.62 &23.79& 42.35& \textbf{43.57} &38.74&33.88&30.37\\
  &&Err-min(S)& 1.00 & 2.98 & 6.61 & 17.87&\textbf{47.10}&44.23 &40.30&34.38&28.86\\
 &&RobustEM & 7.74 &9.37&10.37& 12.58&19.95&\textbf{42.31} &35.42&24.70&18.21\\
  &\multirow{3}{*}{FGSM} & Region-16&1.00 &1.36 & 5.78&2.88&25.35&\textbf{43.63}&40.71 &25.95&37.32\\
  &&Err-min(S)&1.00 &3.50&7.72 & 18.05 &\textbf{46.49}&44.37 &41.44&35.64&25.06\\
  &&RobustEM&7.14 &9.08&10.21&10.98& 15.32&31.02 &\textbf{38.62}&29.32&25.57\\
\bottomrule
\end{tabular}
\end{table}

Table \ref{tab-adv21} demonstrates that unlearnable poisons can be defended by adversarial training with FGSM as inner-maximization method, achieving comparable test accuracy with PGD-10.
This implies that victim models do not require stronger adversarial attacks to defend unlearnable examples. 
All we need is to destroy the injected linear features, and a smaller number of steps of adversarial perturbation can achieve this goal.

\subsection{Defense under different adversarial training methods}
\label{def-power}
Adversarial training is designed to defend adversarial examples \cite{goodfellow2014explaining, miao2022isometric} originally, but we find that AT can defend unlearnable examples as well.
Various adversarial defense methods could also defend unlearnable examples, as indicated by Theorem \ref{th-o20}.
Experiments on various adversarial defense methods are shown in Table \ref{tab-adv-def}.
From  Table \ref{tab-adv-def}, all of the defenses achieve quite high test accuracy on poisoned dataset.
It is worth noting that Adversarial Weight Perturbation (AWP) \cite{wu2020adversarial} achieves the best defense power.

\begin{table}[ht]
\caption{Test accuracy (\%) for several adversarial defense methods for poisoned CIFAR-10 under adversarial defense budget $\epsilon=8/255$.}
\label{tab-adv-def}
\centering
\begin{tabular}{cccccc}
\toprule
  Adversarial Defense & RobustEM & Region-16 & Err-min(S)& NTGA & AR\\
  \midrule
Vallina AT \cite{madry2018towards} &  80.73   &  83.25 &  \textbf{84.00} & 83.76 & 84.11\\
  FGSM-AT \cite{wong2019fast} &  78.92   &  81.25 &  80.36 & 81.38 & 78.58\\
  TRADES \cite{zhang2019theoretically} & 79.25 & 83.10 &  82.89& 79.32 & 82.27\\
  MART \cite{wang2019improving} & 80.95 &  81.21 &  80.56& 82.08 & 81.72\\
  AWP \cite{wu2020adversarial} & \textbf{81.06} &  \textbf{85.78} &  83.12 &\textbf{84.19} & \textbf{85.25}\\
  \bottomrule
\end{tabular}
\end{table}

We give some intuitive explanations below.
When training a network on poisoned data $D^{\poi}$, if the random seed, initialization and training process are determined, the parameters of network after training by algorithm $\A$ is determined as well, denoted as $\A(D^{\poi})$.
Adversarial training on $D^{\poi}$ could make the network perform well on a neighbourhood of each sample in $D^{\poi}$, which includes many samples in clean dataset $D$, if the adversarial budget is not less than poison budget.
Therefore, a good defense algorithm $\A$ will lead to a small population loss $\mathbb{E}_{(\boldsymbol{x},y)\sim \D} L(\F_{\A(D^{\poi})}(\boldsymbol{x}),y)$.

On the other hand, for an effective defense method, the poison method should not be restricted and  thus $D^{\poi}$ may change.
Therefore, the network parameters $\A(D^{\poi})$ may also change when different poisoned dataset $D^{\poi}$ are used, giving rise to the inspiration that a small population loss should be achieved under the weight perturbations regime \cite{liu2017faultinjection, yu2023adversarial}.
We believe that if a network is more robust on weight perturbations, the defense power for unlearnable data poisoning attacks is stronger.
More effective defense methods under weight perturbation regimes and potential theoretical analyses will be left as future works.

\section{Boarder impact statement}
\label{boarder}
In this paper, we propose effective methods for detecting and defending against unlearnable examples, which are designed to preserve privacy.
Although unlearnable examples have not been widely deployed in real-world applications, our work may cause potential users to reconsider their use as a privacy preservation device.
As a privacy protection method, unlearnable examples must undergo extensive detection and defense verification before being deployed in applications.


\end{document}